\documentclass[]{sysu_preprint}

\usepackage[utf8]{inputenc}

\usepackage{url}
\usepackage{longtable}
\newenvironment{ack}{}{}
\usepackage{amsmath}
\usepackage{amssymb}
\usepackage{amsfonts}
\usepackage{mathtools}
\usepackage{amsthm}
\usepackage{nicefrac}       
\usepackage{colortbl}       
\usepackage{xcolor}         
\usepackage{enumitem}       
\usepackage{pifont}         
\usepackage{algorithm}
\usepackage{algorithmic}

\definecolor{ourmethod}{RGB}{230, 243, 255}  
\definecolor{datasetcol}{RGB}{232, 245, 233}  
\definecolor{best}{gray}{0.75}  
\definecolor{second}{gray}{0.88}  

\theoremstyle{plain}
\newtheorem{theorem}{Theorem}[section]
\newtheorem{proposition}[theorem]{Proposition}

\theoremstyle{definition}

\newtheorem{assumption}[theorem]{Assumption}
\crefname{assumption}{Assumption}{Assumptions}
\theoremstyle{remark}

\usepackage[disable,textsize=tiny]{todonotes}

\providecommand{\RETURN}{\STATE \textbf{return} }

\title{CoFlow: Coordinated Few-Step Flow for Offline Multi-Agent Decision Making}

\renewcommand{\authorlist}{%
  \authorfont{\sffamily Guowei Zou\footnotemark[1], Haitao Wang, Beiwen Zhang, Boning Zhang, Hejun Wu}%
}
\renewcommand{\affiliationlist}{\affiliationfont{\sffamily\seedblue{Sun Yat-sen University}}}
\abstract{
How can generative offline multi-agent reinforcement learning achieve both fast joint trajectory generation and effective cooperation? Multi-agent diffusion models generate joint trajectories through iterative sampling but incur high inference latency. Distillation reduces sampling steps but may lose the teacher's cross-agent dependencies, reducing the success rate of cooperative tasks. A joint few-step generator is able to be trained directly to avoid such dependcency lost, but the cooperation of multi-agents is limited by insufficient modeling of cross-agent dependencies. In addition, its consistency training process introduces heavy computational and memory costs. Therefore, we propose \textbf{Co}ordinated few-step \textbf{Flow} (CoFlow), which directly learns a joint averaged-velocity field for distillation-free one-step and few-step multi-agent trajectory generation. To support cooperation during generation, Coordinated Velocity Attention (CVA) incorporates teammate trajectory information into each agent's prediction. To reduce the consistency-training cost of this joint model, we approximate the correction term using finite differences. Across 48 configurations on MPE and SMAC, CoFlow supports centralized and decentralized execution. Under centralized execution, CoFlow outperforms our reproduced baseline by 11.2\%, averaging relative gains equally across both benchmark suites, with a 12.93$\times$ model-sampling speedup. CVA improves normalized scores by 165.1\% on average over the same model with cross-agent attention disabled. Compared with the exact-derivative implementation, finite-difference training achieves a 1.78$\times$ speedup and reduces peak GPU memory by 41.1\%. \textbf{Project page: \url{https://github.com/Guowei-Zou/coflow}.}
}

\hypersetup{
  pdftitle={CoFlow: Coordinated Few-Step Flow for Offline Multi-Agent Decision Making},
  pdfauthor={Guowei Zou, Haitao Wang, Beiwen Zhang, Boning Zhang, Hejun Wu},
  pdfsubject={cs.AI, cs.LG, cs.MA},
  pdfkeywords={multi-agent reinforcement learning, offline reinforcement learning, flow matching, few-step generation}
}

\begin{document}
\maketitle
\footnotetext[1]{\textbf{Emails:} \{zougw, zhangbw39, zhangbn7\}@mail2.sysu.edu.cn; \{wanght76, wuhejun\}@mail.sysu.edu.cn.}

\vskip 0.1in
\begin{center}
\includegraphics[width=0.85\textwidth]{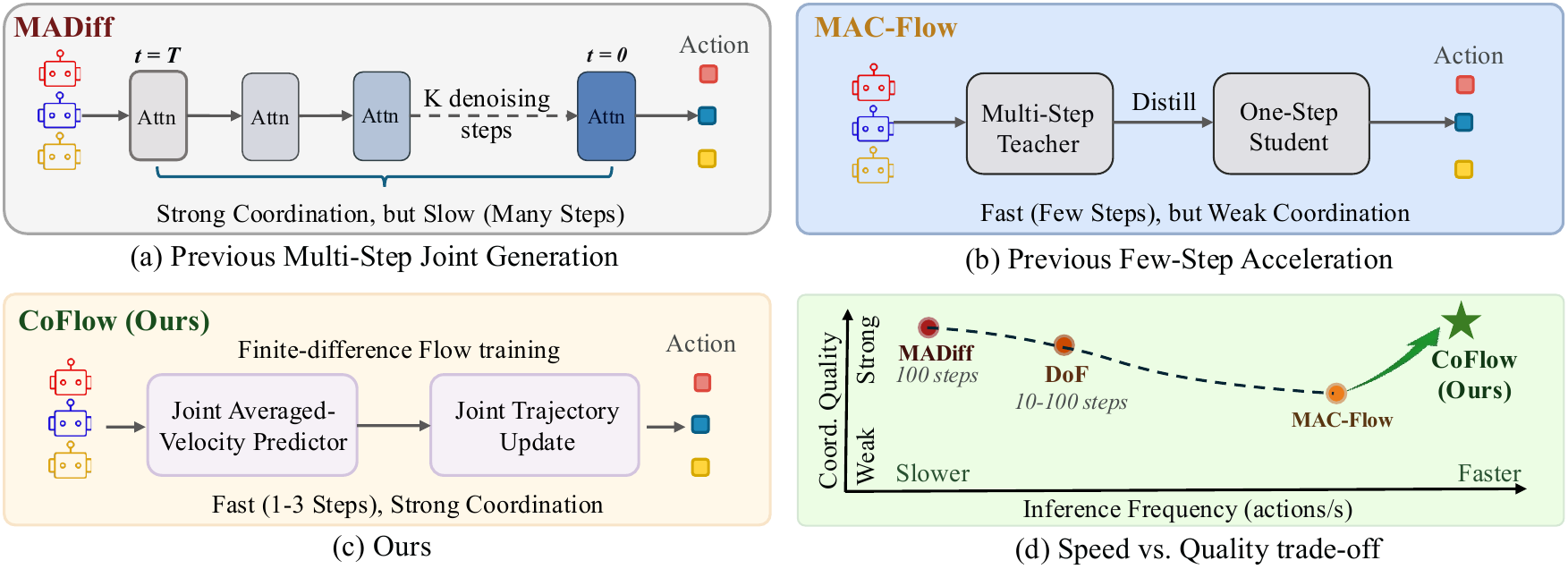}
\end{center}
\vskip -0.05in
{%
\captionof{figure}{The dilemma between quality and efficiency in offline multi-agent trajectory generation. \textbf{(a)}~MADiff-style multi-step joint generation preserves coordination through many denoising steps. \textbf{(b)}~MAC-Flow distills a joint teacher into one-step students, requiring preservation of cross-agent dependencies during compression. \textbf{(c)}~CoFlow combines joint averaged-velocity learning, CVA coupling, and finite-difference training. \textbf{(d)}~CoFlow shifts the trade-off between speed and quality against representative diffusion and acceleration baselines.}
\label{fig:method_comparison}
}%
\vskip 0.15in

\section{Introduction}

Generative policies have become an important paradigm in offline multi-agent reinforcement learning (MARL), modeling cooperative behavior through joint generation of agents' trajectories \citep{zhu2024madiff, yuan2025madits, li2025dof, lu2025diffusion_veteran}. However, existing multi-agent diffusion \citep{ho2020ddpm, song2021score} and flow-matching policies \citep{lipman2023flow, liu2023flow} typically rely on multi-step iterative sampling \citep{zhu2024madiff, li2025dof, lee2025macflow}. As illustrated in \Cref{fig:method_comparison}(a), the model repeatedly updates the joint trajectory to capture cross-agent relationships during generation, but repeated network evaluations incur substantial inference costs \citep{janner2022planning, wang2023diffusion, chi2023diffusion}. Therefore, reducing the number of generation steps while maintaining effective cooperation is a central challenge for generative multi-agent decision making.

To reduce multi-step sampling costs \citep{salimans2022progressive, song2023consistency}, existing approaches distill a joint teacher into per-agent students, enabling generation in fewer steps \citep{lee2025macflow}, as illustrated in \Cref{fig:method_comparison}(b). However, this compression must preserve not only individual trajectory quality but also the cross-agent dependencies learned by the teacher \citep{lee2025macflow}. If these dependencies are lost, trajectories generated independently by the agents may fail to work together effectively, impairing cooperative task performance \citep{barde2024offline, hernandez2019survey}.

To avoid losing dependencies during distillation, one route is to extend existing direct few-step generation methods \citep{frans2025shortcut, dmpo} to joint multi-agent trajectory modeling. These methods learn to generate through one or a few updates without training a separate teacher \citep{geng2025avgvelocity, geng2025fastforward, guo2025intervalsplitting}. However, direct joint learning still faces two requirements. The model must use teammate information to generate coordinated trajectories within a limited number of updates, while consistency training of the joint averaged-velocity model can incur substantial computational and memory costs \citep{geng2025avgvelocity, geng2025fastforward}. Few-step multi-agent generation must therefore address both the modeling of cooperative dependencies and the efficiency of training.

To address these challenges, we propose \textbf{Co}ordinated few-step \textbf{Flow} (\textbf{CoFlow}), a distillation-free framework for offline multi-agent decision making, with its core design illustrated in \Cref{fig:method_comparison}(c). CoFlow applies averaged-velocity generation to joint multi-agent trajectory learning, enabling the model to generate all agents' trajectories through one or a few updates. Within this joint generation framework, we address cooperative modeling and training efficiency. First, \emph{Coordinated Velocity Attention} (CVA) allows each agent to use teammate trajectory information when generating its own trajectory, modeling cross-agent dependencies within every update. Second, we approximate the consistency correction to reduce the training computation and memory required by this joint averaged-velocity model. Together, these components support joint model learning, allowing CoFlow to combine fewer generation steps with cooperative performance and efficient training under both centralized and decentralized execution. \Cref{fig:method_comparison}(d) compares generation speed and task performance across methods, illustrating CoFlow's combined performance on these two dimensions.

Our contributions are threefold. \textbf{(1) Joint few-step generation framework}: we propose CoFlow for offline multi-agent decision making, directly learning a joint trajectory generation model for one-step or few-step generation without teacher distillation. \textbf{(2) Cross-agent cooperative modeling}: we design Coordinated Velocity Attention to let agents use teammate trajectory information within each generation update, modeling the cross-agent dependencies needed for cooperation. \textbf{(3) Efficient consistency training}: we use finite differences to approximate the consistency correction, reducing the per-update computation and memory costs of training the joint averaged-velocity model.

\section{Related Work}

\textbf{Offline multi-agent RL.} Offline MARL aims to learn cooperative policies from fixed logged datasets \citep{levine2020offline}. Dataset and baseline resources support research in this setting \citep{formanek2023otg}, while standardised baselines and evaluation protocols are important for reliable comparisons \citep{formanek2024dispelling}. Value-based methods regularize Q-values to mitigate distribution shift, including MA-ICQ \citep{yang2021believe}, MA-CQL and OMAR \citep{pan2022plan}, and OMIGA \citep{wang2023omiga}. Recent work scales these to larger teams via low-rank interaction structure, in-sample sequential updates, and per-agent score decomposition \citep{zhan2025exploiting, liu2025inspo, qiao2025omsd}. Sequence-modeling baselines such as the multi-agent decision transformer \citep{meng2023madt} and policy-regularization methods \citep{fujimoto2021minimalist, kostrikov2022iql} complete the comparison set. These methods usually parameterize per-agent Gaussian or autoregressive policies, motivating generative models for multimodal joint-action distributions in coordinated behavior.

\textbf{Single-agent few-step / one-step generation.} Generative trajectory models have become the leading paradigm for single-agent offline RL, covering planning, policy learning, energy guidance, and visuomotor control \citep{janner2022planning, wang2023diffusion, hansen2023idql, chen2023sfbc, ren2024dppo, d2ppo, lu2024contrastive, chi2023diffusion}, all rooted in denoising diffusion probabilistic models (DDPM) and score-based generative modeling \citep{ho2020ddpm, song2021score}. Flow matching \citep{lipman2023flow, liu2023flow} reaches parity via Flow Q-learning, energy-weighted flow matching, ReinFlow, and Flow Policy Optimization (FPO) \citep{park2025flowql, zhang2025efm, zhang2025reinflow, mcallister2026fpo}. Reducing the ${\sim}20$ sampling steps these methods require is an active subfield. \emph{Distillation} compresses a multi-step teacher into a one-step student via Consistency Models, Progressive Distillation, and denoising diffusion implicit models (DDIM) \citep{song2023consistency, salimans2022progressive, song2020ddim}. \emph{Direct one-step parameterizations} integrate the trajectory in a single forward pass via averaged velocity fields, shortcut reparameterization, dispersive regularization, and latent / discrete variants \citep{geng2025avgvelocity, frans2025shortcut, dmpo, dao2024latent, gat2024discrete, rombach2022latent}. These methods now match multi-step quality at 1--3 steps in single-agent settings, motivating a multi-agent counterpart with an explicit inter-agent mechanism.

\textbf{Multi-agent few-step / one-step generation.} Existing multi-agent generative methods sit at one of two extremes. The \emph{diffusion family} preserves or factorizes joint generation but inherits multi-step inference cost. MADiff \citep{zhu2024madiff} embeds cross-agent attention into diffusion-based planning, MADiTS \citep{yuan2025madits} stitches high-quality coordination segments, and DoF \citep{li2025dof} factorizes a centralized diffusion via independence-conditioned decomposition. \citet{lu2025diffusion_veteran} find that unconditional sampling with selection often beats guided sampling. The \emph{few-step flow family} reduces sampling cost: MAC-Flow \citep{lee2025macflow} distills a joint flow teacher into per-agent students. CoFlow learns a joint averaged-velocity field directly, combining cross-agent coupling within each generation step with finite-difference consistency training.

\section{Methodology}
\label{sec:method}

\begin{figure*}[!htbp]
\centering
\includegraphics[width=0.95\textwidth]{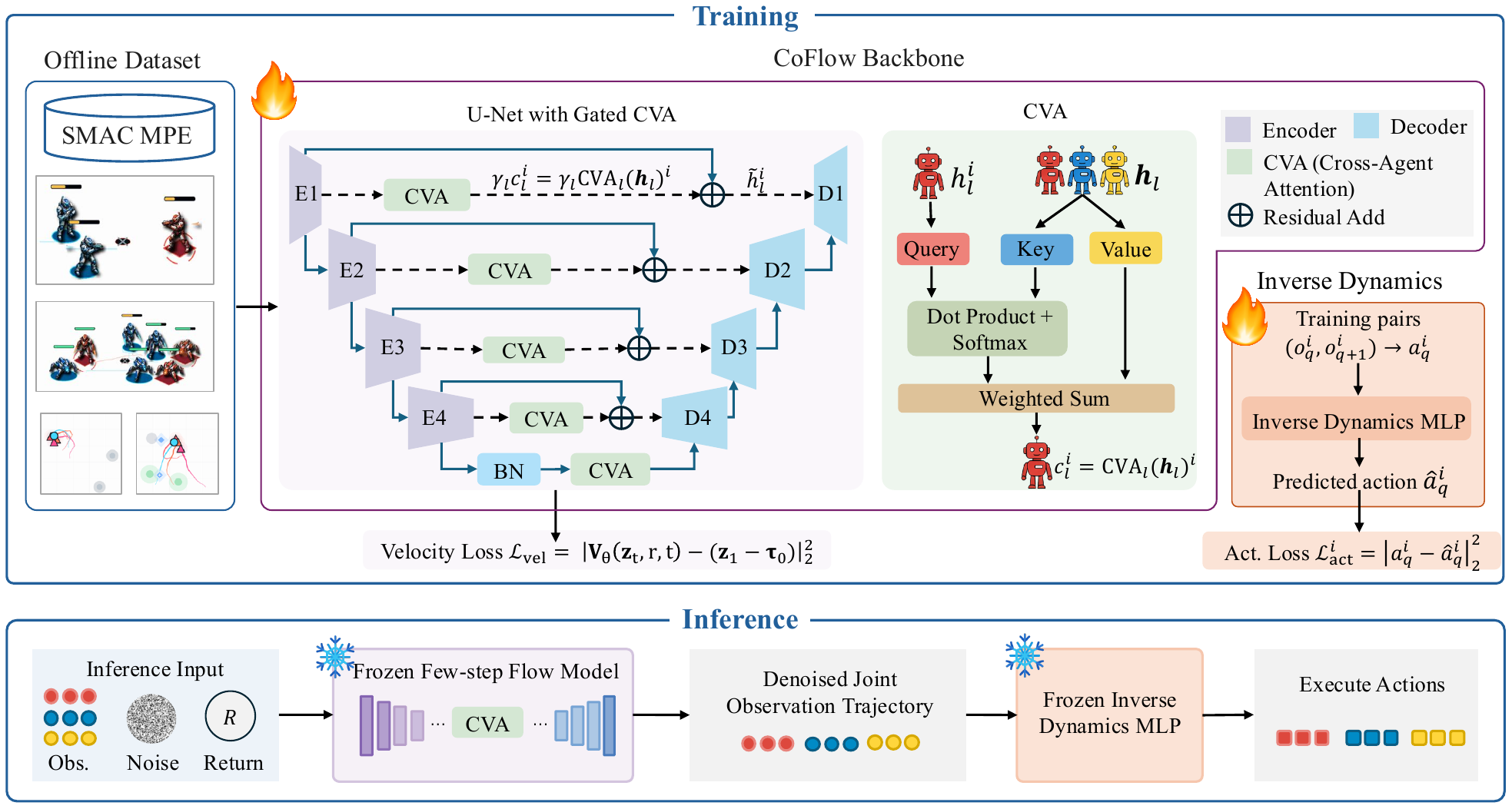}
\caption{CoFlow training and inference overview. During training, the shared U-Net backbone injects gated CVA features into the skip paths to learn a joint velocity field, while a separate inverse-dynamics multilayer perceptron (MLP) is trained from offline observation pairs. During inference, the frozen few-step flow model predicts a denoised joint observation trajectory, and the frozen inverse-dynamics MLP maps the first predicted next observation to executable actions. The loss panels show unnormalized per-example squared errors for illustration; the full objectives in \Cref{eq:velocity_loss,eq:action_loss} include averaging and use legal-action cross-entropy for discrete control.}
\label{fig:framework}
\end{figure*}

CoFlow learns a joint averaged-velocity field for few-step multi-agent trajectory generation. CVA models cross-agent dependencies, while finite-difference consistency training reduces the cost of training this joint field. Training matches the network prediction plus a detached correction to the conditional velocity; inference uses the averaged prediction alone. \Cref{fig:framework} shows the architecture, with background and derivations in \Cref{app:preliminaries,app:method_details_inline}.

\subsection{Joint Averaged-Velocity Generation}

From fixed offline trajectories, conditioned on observation history and target return, we construct the joint training pair (conditioning arguments are suppressed):
\begin{equation}
\label{eq:training_pair}
\mathbf z_t=(1-t)\boldsymbol\tau_0+t\mathbf z_1,
\qquad \mathbf v_c=\mathbf z_1-\boldsymbol\tau_0,
\end{equation}
where $\boldsymbol\tau_0$ is a clean joint observation trajectory, $\mathbf z_1$ is independent standard Gaussian noise, and $\mathbf v_c$ is the sample-conditional velocity. Bold symbols stack $N$ agent blocks of dimension $d$; superscript $i$ selects block $i$, e.g., $z_t^i=[\mathbf z_t]^i$. Flow time runs from data ($t=0$) to noise ($t=1$).

CoFlow learns the joint interval-averaged velocity
\begin{equation}
\label{eq:average_velocity}
\mathbf u(\mathbf z,r,t)
=\frac{\mathbf z-\Phi_{r\leftarrow t}(\mathbf z)}{t-r},
\qquad r<t,
\end{equation}
where $r$ is the reference time and $\Phi_{r\leftarrow t}$ transports the joint state from $t$ to $r$. The field $\mathbf u$ and its diagonal limit, the instantaneous marginal velocity $\mathbf v$, are theoretical quantities; their network approximations carry subscript $\theta$.

We learn this interval displacement through conditional-velocity supervision and the MeanFlow consistency relation \citep{geng2025avgvelocity}.

\subsection{Gated Coordinated Velocity Attention}

We use one shared temporal U-Net for the agent streams. CVA attends across agents and updates each encoder skip before decoder fusion:
\begin{equation}
\label{eq:cva}
c_l^i=\mathrm{CVA}_l(\boldsymbol h_l)^i,
\qquad \tilde h_l^i=h_l^i+\gamma_l c_l^i,
\end{equation}
where $h_l^i,\tilde h_l^i$ are agent $i$'s layer-$l$ skips before/after CVA, $c_l^i$ is the message, and $\gamma_l$ is a zero-initialized scalar gate.

Shared multi-head attention incorporates the U-Net conditioning embeddings. The decoder fuses the enhanced skips with its current features, and the bottleneck uses the same gated operation. Zero initialization gives an independent backbone; learned gates activate teammate information within each evaluation. Thus $u_\theta^i=[\mathbf u_\theta]^i$ can depend on the full joint input. See \Cref{app:decomposition} for the decomposition.

\subsection{Finite-Difference Consistency Training}

Here, consistency training enforces the relationship between averaged and instantaneous velocities: the averaged velocity plus its interval-scaled derivative along the flow trajectory should match the instantaneous velocity. We use finite differences to approximate this derivative when constructing the training correction.

Along a marginal flow trajectory, $(t-r)\mathbf u(\mathbf x_t,r,t)=\mathbf x_t-\mathbf x_r$. Differentiating at fixed $r$ yields
\begin{equation}
\label{eq:displacement_derivative}
\frac{d}{dt}\left[(t-r)\mathbf u(\mathbf x_t,r,t)\right]
=\frac{d}{dt}(\mathbf x_t-\mathbf x_r)=\mathbf v(\mathbf x_t,t),
\end{equation}
where $\mathbf x_t$ follows the marginal ordinary differential equation (ODE) driven by $\mathbf v$, rather than a sampled linear interpolation. The reference state $\mathbf x_r$ is fixed.

The product and chain rules give the joint MeanFlow identity:
\begin{equation}
\label{eq:meanflow_main}
\begin{aligned}
\mathbf u+(t-r)\frac{d\mathbf u}{dt}&=\mathbf v,\\
\frac{d\mathbf u}{dt}&=(\nabla_{\mathbf x}\mathbf u)\mathbf v+\partial_t\mathbf u,
\end{aligned}
\end{equation}
where $\nabla_{\mathbf x}\mathbf u$ is the full joint Jacobian, including cross-agent blocks. All fields are evaluated along $\mathbf x_t$. The partial derivative $\partial_t$ is taken with $r$ and the state held fixed.

Using $\mathbf v_c$ as the joint training tangent, we approximate the derivative using a finite difference (FD):
\begin{equation}
\label{eq:fd_correction}
\boldsymbol\delta_{\mathrm{FD}}\coloneqq\mathrm{stopgrad}\!\left[
\frac{\mathbf u_\theta(\mathbf z_t+\varepsilon\mathbf v_c,r,t+\varepsilon)
-\mathbf u_\theta(\mathbf z_t,r,t)}{\varepsilon}\right],
\end{equation}
where $\mathbf u_\theta$ predicts the joint averaged velocity, $\theta$ includes the gates, and the signed step satisfies $\varepsilon\ne0$ and $r\le t+\varepsilon\le1$. The data-derived tangent $\mathbf v_c$ is held fixed. Shifting state and time together follows the direction $(\mathbf v_c,1)$ in \Cref{eq:meanflow_main}. Since $\mathbf v(\mathbf z_t,t)=\mathbb E[\mathbf v_c\mid\mathbf z_t,t]$, sampled tangents yield a stochastic surrogate.

We therefore match the averaged prediction plus its interval-scaled correction to the instantaneous-velocity target:
\begin{equation}
\label{eq:consistency_surrogate}
\mathbf V_\theta
\coloneqq\mathbf u_\theta(\mathbf z_t,r,t)+(t-r)\boldsymbol\delta_{\mathrm{FD}},
\end{equation}
\begin{equation}
\label{eq:velocity_loss}
\mathcal L_{\mathrm{vel}}(\theta)
=\frac{1}{Nd}\mathbb E\left\|\mathbf V_\theta-\mathbf v_c\right\|_2^2,
\end{equation}
where uppercase $\mathbf V_\theta$ is the corrected training prediction and lowercase $\mathbf v_c$ is the data-derived target. The expectation averages trajectories, noise, and time pairs, with coordinate normalization $1/(Nd)$.

With $\mathrm{stopgrad}$, only the base prediction receives gradients; its value is reused in the quotient, and the shifted query needs no backward graph. At $r=t$, the correction vanishes and the shifted query is skipped. The finite-difference correction is evaluated on the CVA-coupled joint field, so it incorporates cross-agent interactions within the same trajectory update. The Jacobian-vector product (JVP) relation and Taylor bound are in \Cref{app:method_details_inline,app:fd_jvp_equiv}.

\subsection{Action Decoding and Inference}

Because the flow generates observation trajectories, a separate three-layer inverse-dynamics multilayer perceptron (MLP) maps consecutive observations to executable actions. Its loss and the joint objective are
\begin{equation}
\label{eq:action_loss}
\mathcal L_{\mathrm{act}}(\phi)
=\frac1N\sum_{i=1}^N\mathbb E\!\left[\ell_{\mathrm{act}}\!\left(
a_q^i,I_\phi^i(o_q^i,o_{q+1}^i);\mathcal U_q^i\right)\right],
\end{equation}
\begin{equation}
\label{eq:total_loss}
\mathcal L(\theta,\phi)
=\mathcal L_{\mathrm{vel}}(\theta)+\lambda\mathcal L_{\mathrm{act}}(\phi),
\qquad\lambda=1,
\end{equation}
where $I_\phi^i$ is the inverse-dynamics predictor with parameters $\phi$, $o_q^i,a_q^i$ denote observations and actions at environment step $q$, and $\lambda$ weights action learning. The loss $\ell_{\mathrm{act}}$ is continuous-action mean squared error (MSE) or cross-entropy over legal actions $\mathcal U_q^i$.

The interval length converts averaged velocity to displacement; subtraction moves from noise toward data. Starting at $\mathbf z_{t_K}=\mathbf z_1$, CoFlow runs $K$ joint updates, then decodes each agent's action from its output block:
\begin{equation}
\label{eq:coflow_inference}
\begin{aligned}
\mathbf z_{t_{k-1}}&=\mathbf z_{t_k}-(t_k-t_{k-1})
\hat{\mathbf u}_\theta(\mathbf z_{t_k},t_{k-1},t_k),\\
a_q^i&=\operatorname{Dec}_i\!\left(I_\phi^i(o_q^i,\hat o_{q+1}^i);\mathcal U_q^i\right),
\end{aligned}
\end{equation}
where $1=t_K>\cdots>t_0=0$, $k=K,\ldots,1$, and $\hat{\mathbf u}_\theta$ combines return-conditioned and return-dropped predictions for guidance. The next observation $\hat o_{q+1}^i$ is read from block $i$ of the generated trajectory. The decoder returns a continuous action or the highest-logit legal action.

For $K=1$, the generated trajectory is $\hat{\boldsymbol\tau}_0=\mathbf z_1-\hat{\mathbf u}_\theta(\mathbf z_1,0,1)$; the correction is used only in training. CoFlow-C (centralized execution) uses joint observations to generate a joint trajectory and decode all agents' actions. CoFlow-D (decentralized execution) performs joint trajectory prediction conditioned on each agent's local observations, with unavailable teammate context masked. Cross-agent attention models teammate behavior within each local prediction, and each agent executes only its own decoded action. Guidance and algorithms are in \Cref{app:method_details_inline,app:algos}; proofs are in \Cref{app:fd_jvp_equiv,app:decomposition,app:wasserstein}.

\refstepcounter{section}\label{sec:experiments}
\par\vspace{0.35ex}
\noindent{\large\sc\raggedright \thesection\quad Experiments\par}
\vspace{0.35ex}

We organize the evaluation around three linked research questions ({\color{orange}RQ}):

{\color{orange}\textbf{RQ1.}} \textbf{How well does CoFlow perform on continuous and discrete MARL benchmarks?}

{\color{orange}\textbf{RQ2.}} \textbf{How much reward gain is routed through CVA versus per-agent capacity?}

{\color{orange}\textbf{RQ3.}} \textbf{When does single-pass inference suffice, and which design choice drives it?}

More detailed visualization experiments, ablation studies, and analyses are provided in \Cref{app:experiments}.

\begin{table*}[!tbp]
\caption{Main results across MPE and SMAC. CoFlow entries report mean $\pm$ standard deviation across three random seeds using 1--3 sampling steps (\Cref{sec:exp_setup}). MPE uses OMAR-normalized scores; SMAC uses episodic returns. \textit{Dataset} gives the offline-data mean. Bold and underline mark the best and second-best scores before display rounding.}
\label{tab:main_results}
\centering
\fontsize{7.5pt}{9pt}\selectfont
\renewcommand{\arraystretch}{0.9}

\par\vspace{0.35em}\noindent\textbf{(a) MPE, OMAR-normalized score ($\uparrow$)}\par\vspace{0.25em}
\setlength{\tabcolsep}{2.6pt}%
\begin{tabular*}{\textwidth}{@{\extracolsep{\fill}} l>{\columncolor{datasetcol}}l >{\columncolor{datasetcol}}c cccccccc}
\toprule
 & \multicolumn{2}{c}{\cellcolor{datasetcol}\textit{Dataset}} & \multicolumn{5}{c}{\textit{Gaussian / Value-based}} & \textit{Diffusion} & \multicolumn{2}{c}{\textit{Flow (Ours)}} \\
Task & \cellcolor{datasetcol}Quality & \cellcolor{datasetcol}Data & BC & MA-ICQ & MA-TD3+BC & MA-CQL & OMAR & MADiff-C & CoFlow-C & CoFlow-D \\
\midrule
\multirow{4}{*}{Spread} & Expert & 96.2 & 35.0{\tiny$\pm$2.6} & 104.0{\tiny$\pm$3.4} & 108.3{\tiny$\pm$3.9} & 98.2{\tiny$\pm$5.2} & \underline{114.9{\tiny$\pm$2.6}} & 95.0{\tiny$\pm$5.3} & \textbf{114.9{\tiny$\pm$3.5}} & 100.0{\tiny$\pm$10.5} \\
 & Md-Replay & 9.3 & 10.0{\tiny$\pm$3.8} & 13.6{\tiny$\pm$5.7} & 15.4{\tiny$\pm$5.6} & 31.4{\tiny$\pm$7.2} & \underline{37.9{\tiny$\pm$6.1}} & 30.3{\tiny$\pm$2.5} & \textbf{42.9{\tiny$\pm$5.3}} & 24.0{\tiny$\pm$3.2} \\
 & Medium & 27.0 & 31.6{\tiny$\pm$4.8} & 29.3{\tiny$\pm$5.5} & 39.4{\tiny$\pm$3.6} & 34.1{\tiny$\pm$7.2} & 47.9{\tiny$\pm$18.9} & \textbf{64.9{\tiny$\pm$7.7}} & \underline{57.6{\tiny$\pm$1.2}} & 48.9{\tiny$\pm$2.9} \\
 & Random & 0.1 & $-$0.5{\tiny$\pm$3.2} & 6.3{\tiny$\pm$3.5} & 9.8{\tiny$\pm$4.9} & 24.0{\tiny$\pm$9.8} & \underline{34.4{\tiny$\pm$5.3}} & 6.9{\tiny$\pm$3.1} & \textbf{45.5{\tiny$\pm$8.5}} & 10.8{\tiny$\pm$5.0} \\
\midrule
\multirow{4}{*}{Tag} & Expert & 89.5 & 40.0{\tiny$\pm$9.6} & 113.0{\tiny$\pm$14.4} & 115.2{\tiny$\pm$12.8} & 119.3{\tiny$\pm$14.0} & 123.9{\tiny$\pm$10.5} & \textbf{168.8{\tiny$\pm$13.1}} & \underline{137.2{\tiny$\pm$8.7}} & 124.8{\tiny$\pm$2.4} \\
 & Md-Replay & 8.6 & 0.9{\tiny$\pm$1.4} & 34.5{\tiny$\pm$27.8} & 28.7{\tiny$\pm$20.9} & 41.7{\tiny$\pm$15.3} & 47.1{\tiny$\pm$15.3} & \textbf{98.8{\tiny$\pm$11.3}} & 66.8{\tiny$\pm$1.7} & \underline{75.0{\tiny$\pm$12.0}} \\
 & Medium & 43.8 & 22.5{\tiny$\pm$1.8} & 63.3{\tiny$\pm$20.0} & 65.1{\tiny$\pm$29.5} & 61.7{\tiny$\pm$23.1} & 66.7{\tiny$\pm$23.2} & \textbf{133.5{\tiny$\pm$20.2}} & \underline{110.3{\tiny$\pm$6.5}} & 86.0{\tiny$\pm$10.5} \\
 & Random & 0.0 & 1.2{\tiny$\pm$0.5} & 2.2{\tiny$\pm$1.3} & 6.0{\tiny$\pm$2.1} & \underline{11.1{\tiny$\pm$2.8}} & \underline{11.1{\tiny$\pm$2.8}} & 10.6{\tiny$\pm$3.7} & \textbf{30.6{\tiny$\pm$6.0}} & 7.0{\tiny$\pm$1.9} \\
\midrule
\multirow{4}{*}{World} & Expert & 93.8 & 33.0{\tiny$\pm$9.9} & 109.5{\tiny$\pm$22.8} & 110.3{\tiny$\pm$21.3} & 119.8{\tiny$\pm$28.1} & 110.4{\tiny$\pm$25.7} & 109.3{\tiny$\pm$15.4} & \textbf{156.6{\tiny$\pm$9.5}} & \underline{128.1{\tiny$\pm$10.0}} \\
 & Md-Replay & 11.2 & 2.3{\tiny$\pm$1.5} & 12.0{\tiny$\pm$9.1} & 17.4{\tiny$\pm$8.1} & 19.3{\tiny$\pm$18.3} & \underline{42.9{\tiny$\pm$19.5}} & 19.8{\tiny$\pm$6.2} & \textbf{43.6{\tiny$\pm$10.4}} & 42.6{\tiny$\pm$4.6} \\
 & Medium & 54.5 & 25.3{\tiny$\pm$2.0} & 71.9{\tiny$\pm$20.0} & 73.4{\tiny$\pm$9.3} & 58.6{\tiny$\pm$11.2} & 74.6{\tiny$\pm$11.5} & 84.7{\tiny$\pm$12.3} & \textbf{128.8{\tiny$\pm$12.8}} & \underline{115.3{\tiny$\pm$7.2}} \\
 & Random & 0.0 & $-$2.4{\tiny$\pm$0.5} & 1.0{\tiny$\pm$3.2} & 2.8{\tiny$\pm$5.5} & 0.6{\tiny$\pm$2.0} & 5.9{\tiny$\pm$5.2} & \underline{6.1{\tiny$\pm$2.4}} & 1.3{\tiny$\pm$3.1} & \textbf{23.5{\tiny$\pm$12.7}} \\
\bottomrule
\end{tabular*}

\par\vspace{0.35em}\noindent\textbf{(b) SMAC, episode return ($\uparrow$)}\par\vspace{0.25em}
\setlength{\tabcolsep}{2.0pt}%
\begin{tabular*}{\textwidth}{@{\extracolsep{\fill}} l>{\columncolor{datasetcol}}l >{\columncolor{datasetcol}}c ccccccccc}
\toprule
 & \multicolumn{2}{c}{\cellcolor{datasetcol}\textit{Dataset}} & \multicolumn{3}{c}{\textit{Gaussian / Value-based}} & \textit{Transformer} & \textit{Diffusion} & \multicolumn{2}{c}{\textit{Flow}} & \multicolumn{2}{c}{\textit{Flow (Ours)}} \\
Task & \cellcolor{datasetcol}Quality & \cellcolor{datasetcol}Mean & BC & MA-ICQ & MA-CQL & MADT & MADiff-C & DoF & MAC-Flow & CoFlow-C & CoFlow-D \\
\midrule
\multirow{3}{*}{3m} & Good & 16.5 & 16.0{\tiny$\pm$1.0} & 18.8{\tiny$\pm$0.6} & 19.0{\tiny$\pm$0.3} & \underline{19.6{\tiny$\pm$0.7}} & 19.3{\tiny$\pm$0.5} & \textbf{19.8{\tiny$\pm$0.2}} & \textbf{19.8{\tiny$\pm$0.2}} & 18.8{\tiny$\pm$4.2} & 17.3{\tiny$\pm$0.8} \\
 & Medium & 10.0 & 8.2{\tiny$\pm$0.8} & 18.1{\tiny$\pm$0.7} & \underline{18.9{\tiny$\pm$0.7}} & 17.2{\tiny$\pm$0.7} & 16.4{\tiny$\pm$2.6} & 18.6{\tiny$\pm$1.2} & 18.0{\tiny$\pm$3.2} & \textbf{20.0{\tiny$\pm$4.2}} & 12.8{\tiny$\pm$0.9} \\
 & Poor & 4.7 & 4.4{\tiny$\pm$0.1} & \textbf{14.4{\tiny$\pm$1.2}} & 5.8{\tiny$\pm$0.4} & 8.9{\tiny$\pm$0.3} & 10.3{\tiny$\pm$6.1} & \underline{10.9{\tiny$\pm$1.1}} & 10.6{\tiny$\pm$2.2} & 8.1{\tiny$\pm$1.9} & 10.4{\tiny$\pm$3.1} \\
\midrule
\multirow{3}{*}{2s3z} & Good & 18.3 & 18.2{\tiny$\pm$0.4} & \underline{19.6{\tiny$\pm$0.3}} & 19.1{\tiny$\pm$0.8} & 19.4{\tiny$\pm$0.1} & 15.9{\tiny$\pm$1.2} & 18.5{\tiny$\pm$0.8} & 19.5{\tiny$\pm$0.5} & \textbf{20.0{\tiny$\pm$1.3}} & 18.3{\tiny$\pm$3.5} \\
 & Medium & 12.6 & 14.3{\tiny$\pm$0.7} & 17.2{\tiny$\pm$0.8} & 14.3{\tiny$\pm$2.0} & 17.4{\tiny$\pm$0.3} & 15.6{\tiny$\pm$0.3} & 18.1{\tiny$\pm$0.9} & 17.6{\tiny$\pm$0.6} & \underline{18.7{\tiny$\pm$1.8}} & \textbf{18.8{\tiny$\pm$0.9}} \\
 & Poor & 6.9 & 6.7{\tiny$\pm$0.3} & \textbf{12.1{\tiny$\pm$0.4}} & \underline{10.1{\tiny$\pm$0.7}} & 9.9{\tiny$\pm$0.2} & 8.5{\tiny$\pm$1.3} & 10.0{\tiny$\pm$1.1} & 8.5{\tiny$\pm$0.6} & 9.8{\tiny$\pm$0.6} & 10.1{\tiny$\pm$0.8} \\
\midrule
\multirow{3}{*}{5m\_vs\_6m} & Good & 16.6 & 16.6{\tiny$\pm$0.6} & 16.3{\tiny$\pm$0.9} & 13.8{\tiny$\pm$3.1} & 18.0{\tiny$\pm$1.0} & 16.5{\tiny$\pm$2.8} & 17.7{\tiny$\pm$1.1} & \textbf{18.6{\tiny$\pm$3.5}} & \underline{18.3{\tiny$\pm$0.7}} & 13.9{\tiny$\pm$1.2} \\
 & Medium & 12.6 & 14.2{\tiny$\pm$0.5} & 17.2{\tiny$\pm$0.4} & 16.8{\tiny$\pm$3.1} & \underline{17.5{\tiny$\pm$0.4}} & 15.2{\tiny$\pm$2.6} & 16.2{\tiny$\pm$0.9} & 15.6{\tiny$\pm$1.3} & \textbf{17.8{\tiny$\pm$0.9}} & 15.2{\tiny$\pm$2.8} \\
 & Poor & 7.5 & 7.5{\tiny$\pm$0.2} & 9.4{\tiny$\pm$0.4} & 10.4{\tiny$\pm$1.0} & 8.9{\tiny$\pm$0.3} & 8.9{\tiny$\pm$1.3} & 10.8{\tiny$\pm$0.3} & 9.8{\tiny$\pm$2.1} & \textbf{12.5{\tiny$\pm$1.2}} & \underline{10.9{\tiny$\pm$0.9}} \\
\midrule
\multirow{3}{*}{8m} & Good & 16.9 & 16.7{\tiny$\pm$0.4} & 19.6{\tiny$\pm$0.3} & 13.1{\tiny$\pm$6.1} & 19.2{\tiny$\pm$0.1} & 18.9{\tiny$\pm$1.1} & 19.6{\tiny$\pm$0.3} & \underline{19.7{\tiny$\pm$0.3}} & \textbf{20.0{\tiny$\pm$1.2}} & \textbf{20.0{\tiny$\pm$0.5}} \\
 & Medium & 10.1 & 10.7{\tiny$\pm$0.5} & \underline{18.6{\tiny$\pm$0.5}} & 16.3{\tiny$\pm$3.1} & 18.0{\tiny$\pm$0.5} & 16.8{\tiny$\pm$1.6} & \underline{18.6{\tiny$\pm$0.8}} & \textbf{19.4{\tiny$\pm$0.6}} & 17.5{\tiny$\pm$0.8} & 17.4{\tiny$\pm$0.9} \\
 & Poor & 5.3 & 5.3{\tiny$\pm$0.1} & 10.8{\tiny$\pm$0.8} & 4.6{\tiny$\pm$2.4} & 5.1{\tiny$\pm$0.1} & 9.8{\tiny$\pm$0.9} & \textbf{12.0{\tiny$\pm$1.2}} & \underline{11.5{\tiny$\pm$0.8}} & 6.8{\tiny$\pm$0.3} & 6.0{\tiny$\pm$0.3} \\
\bottomrule
\end{tabular*}

\end{table*}

\FloatBarrier

\noindent\textbf{Benchmarks.} We evaluate on two standard offline MARL suites: Multi-Agent Particle Environment (MPE), with Spread / Tag / World \citep{mordatch2018emergence,pan2022plan}, and StarCraft Multi-Agent Challenge (SMAC), with 3m / 2s3z / 5m\_vs\_6m / 8m \citep{samvelyan2019starcraft,formanek2023otg}. We follow the dataset splits and evaluation protocol of MADiff \citep{zhu2024madiff}; full environment details are in \Cref{app:experiments}.

\textbf{Baselines.} We compare against value-based and Gaussian-policy methods (BC, MA-ICQ, MA-CQL, MA-TD3+BC, OMAR) \citep{pomerleau1988alvinn,yang2021believe,kumar2020conservative,fujimoto2021minimalist,pan2022plan}, transformer policies (MADT) \citep{meng2023madt}, diffusion models (MADiff, MADiTS, DoF) \citep{zhu2024madiff,yuan2025madits,li2025dof}, and flow methods (MAC-Flow) \citep{lee2025macflow}. MADiff-C denotes centralized execution; its results are obtained by running the official MADiff implementation under our evaluation protocol. Results for the remaining baselines are taken from their respective papers.

\begin{figure*}[!t]
\centering
\includegraphics[width=\textwidth]{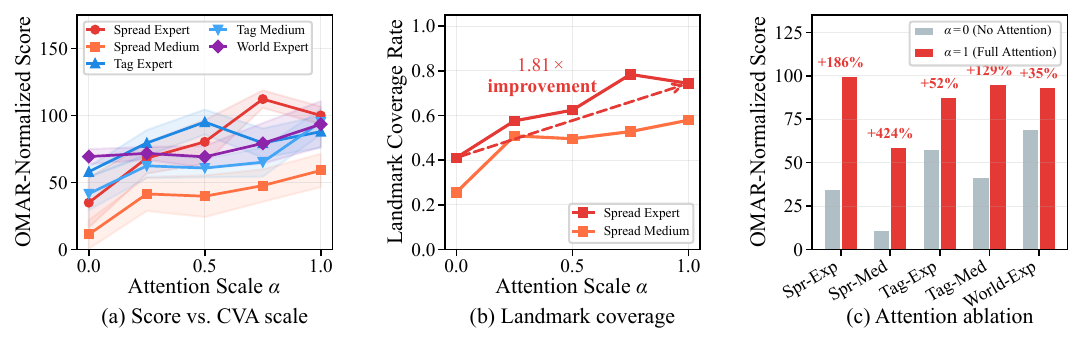}
\caption{Effect of CVA coupling on coordination quality in MPE. Panel \textbf{(a)} shows OMAR-normalized scores across CVA scales, panel \textbf{(b)} reports landmark coverage, and panel \textbf{(c)} summarizes the $\alpha=0$ and $\alpha=1$ endpoints from panel \textbf{(a)}.}
\label{fig:coordination_evidence}
\par\vspace{3pt}

\includegraphics[width=\textwidth]{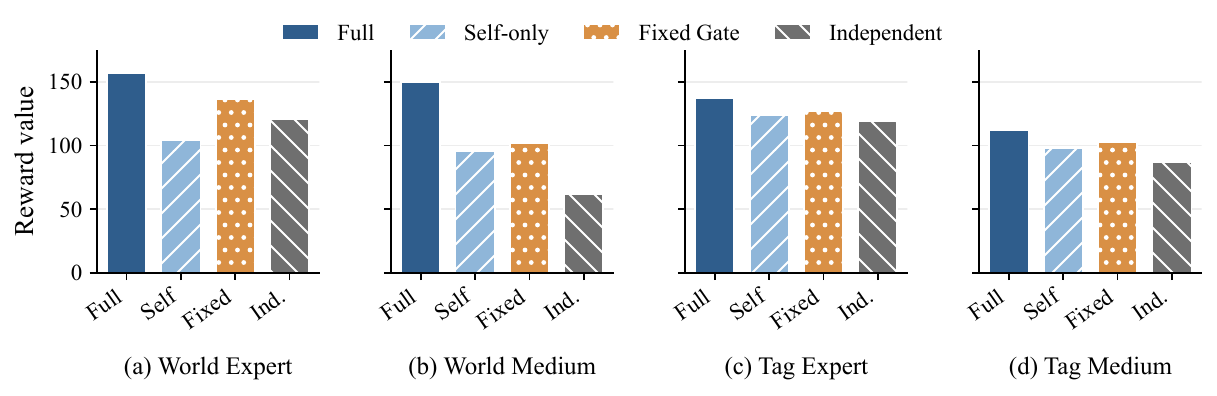}
\caption{Effect of CVA design choices on MPE performance. Each panel compares full CVA, self-only attention, fixed gating, and independent-agent generation using the same OMAR-normalized score metric.}
\label{fig:cva_mechanism_ablation}
\par\vspace{3pt}

\includegraphics[width=\textwidth]{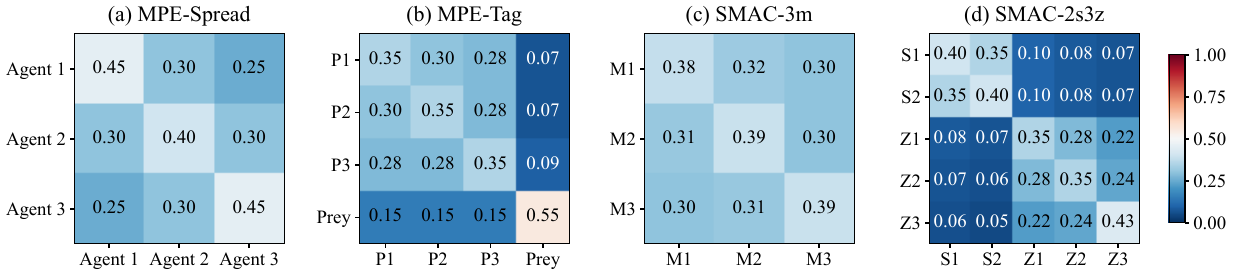}
\caption{CVA attention in representative CoFlow-C runs under centralized execution: MPE-Spread, MPE-Tag, SMAC-3m, and SMAC-2s3z. Rows and columns index agents; the shared color scale indicates feature weights.}
\label{fig:attention_weights}
\end{figure*}

\begin{figure*}[!t]
\centering
\includegraphics[width=\textwidth]{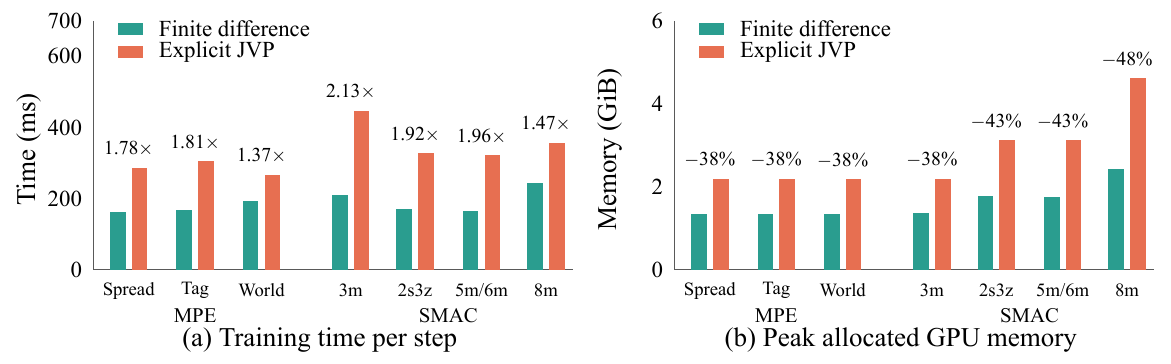}
\caption{Per-update cost of finite difference and explicit JVP across all seven MPE/SMAC task shapes on an A100. We report mean CUDA-event time and peak allocated GPU memory (batch 32, horizon 24; 5 warm-up and 20 measured updates). Here 5m/6m denotes 5m\_vs\_6m.}
\label{fig:training_compute}
\par\vspace{3pt}

\includegraphics[width=\textwidth]{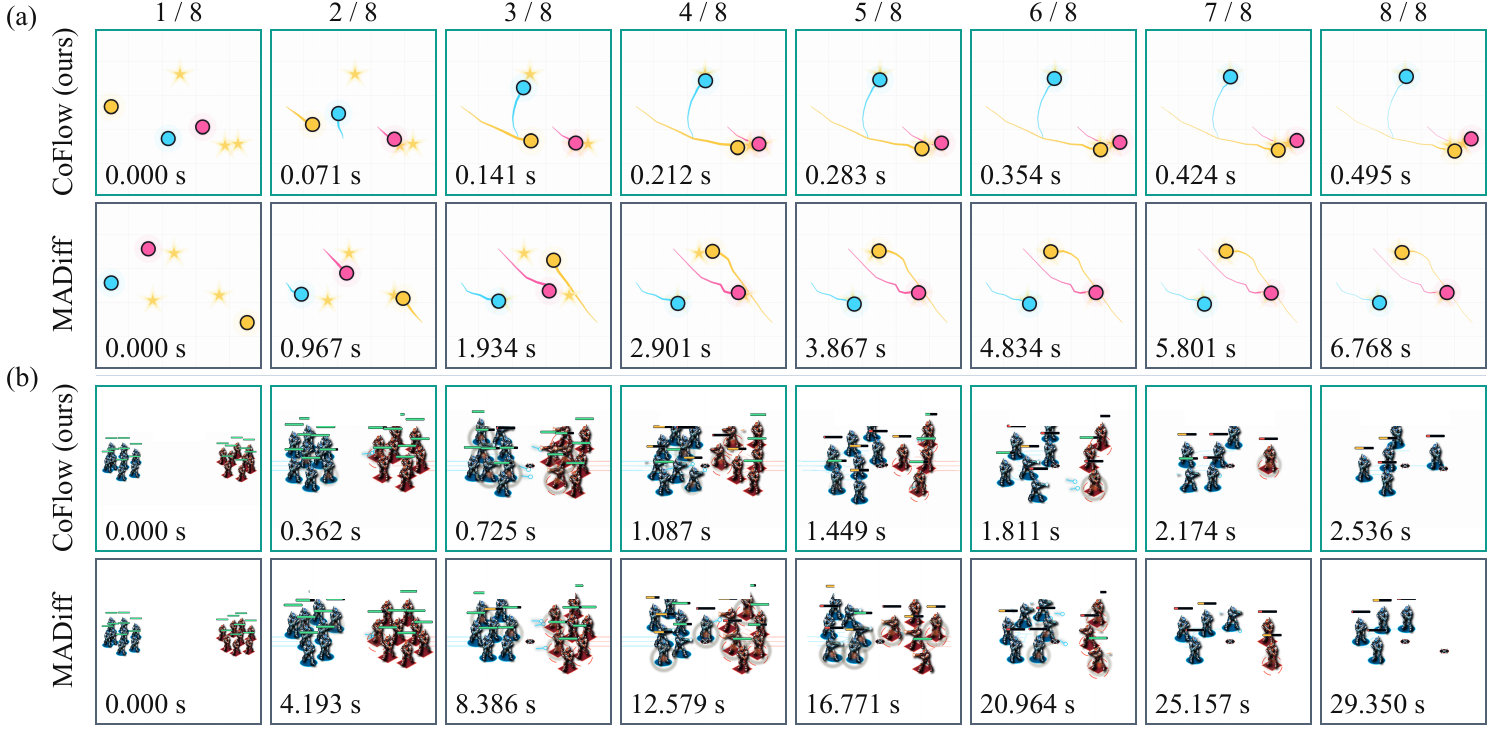}
\caption{Policy rollouts with accumulated model-sampling time: \textbf{(a)}~MPE-Spread (landmark coverage) and \textbf{(b)}~SMAC-8m (eight-versus-eight marine combat). Each row pair shows CoFlow and MADiff rollouts, respectively, with their corresponding model-sampling times.}
\label{fig:model_time_keyframes_main}
\end{figure*}

\subsection{Setup}
\label{sec:exp_setup}

\textbf{CoFlow variants and protocol.} CoFlow and CoFlow-base share the same CVA architecture. CoFlow learns joint averaged velocities through finite-difference consistency training, whereas CoFlow-base uses ordinary conditional flow matching to learn instantaneous velocities. Both support one or multiple sampling steps. At one step, CoFlow-base directly applies a single Euler update across the full noise-to-data interval, without an averaged-velocity training objective. We evaluate both variants with the same execution suffixes: CoFlow-C and CoFlow-base-C use centralized execution, while CoFlow-D and CoFlow-base-D use decentralized execution. MADiff-C scores and uncertainties come from our reproduction; those of the remaining baselines follow the source papers. For CoFlow-C/D, \Cref{tab:main_results} reports mean performance and standard deviation across three random seeds, using 1--3 sampling steps. Curves are shaded by $\pm 1$ standard error. For MPE, we report OMAR-normalized score and apply the same transformation to the offline-data mean and CoFlow returns; the normalization anchors are given in \Cref{app:experiments}.

\subsection{Performance (RQ1)}
\label{sec:exp_performance}

\Cref{tab:main_results} reports the main comparison (best per row in \textbf{bold}, second best \underline{underlined}). Compared with MADiff-C, CoFlow-C achieves an average relative performance improvement of 11.2\%, with the two benchmark suites weighted equally. We first compare the mean scores within each suite and then average the two relative improvements. CoFlow-C shows strong aggregate performance, with particularly strong results on coordination-heavy MPE tasks. CoFlow-D also delivers competitive scores under decentralized execution. The row-wise results characterize performance across tasks and dataset qualities. These coordination-heavy wins motivate the CVA mechanism study in \S\ref{sec:coordination_evidence}.

\subsection{Coordination preservation (RQ2)}
\label{sec:coordination_evidence}

We evaluate CVA as the coordination mechanism by sweeping its contribution through a gating factor $\alpha \in [0, 1]$, with $\alpha{=}0$ giving the attention-disabled setting and $\alpha{=}1$ using the full module. \Cref{fig:coordination_evidence} reports three signals. Full attention ($\alpha{=}1$) gives a higher OMAR-normalized score than the attention-disabled setting ($\alpha{=}0$) on the five displayed MPE configurations, averaging a 165.1\% relative score improvement and reaching 424\% on Spread Medium (panels a and c). Landmark coverage, the fraction of landmarks covered by at least one agent at each timestep, increases from $\alpha{=}0$ to $\alpha{=}1$ on the two Spread configurations (panel b). The largest gains occur on the coordination-heavy Spread settings (panel c). Coverage exposes a coordination failure that score alone can hide: at low $\alpha$, agents cluster on the same landmark. These observations are consistent with cross-agent attention being the main source of the performance change, rather than only per-agent model capacity.

\Cref{fig:cva_mechanism_ablation} contrasts CVA design choices across four MPE configurations. Self-only restricts attention to each agent's own features; fixed gating removes gate adaptation; independent generation processes agents separately. Full CVA gives the highest reported World and Tag scores, consistent with benefits from teammate conditioning and adaptive gating. All variants are compared under the same experimental configuration using OMAR-normalized scores.

We then inspect the learned CVA structure. The attention weight matrices in \Cref{fig:attention_weights} differ across tasks using the same architecture and role-agnostic inputs: MPE-Spread shows near-uniform mixing, MPE-Tag concentrates attention among predators while assigning weaker attention to the prey, and SMAC-2s3z produces a block-diagonal Stalker/Zealot grouping. CVA thus operates as \emph{selective} inter-agent information sharing with task-dependent mixing. The displayed configurations in \Cref{fig:gamma_signs} exhibit small scalar gate magnitudes ($\max_l |\gamma_l| \le 0.15$). The magnitude of the cross-agent correction also depends on the norms of the corresponding projection matrices. Under its stated assumptions, \Cref{thm:decomposition} characterizes their joint effect through the combined scale of the gates and projection matrices.

\subsection{Few-step sufficiency (RQ3)}
\label{sec:exp_fewstep}

The appendix heatmap in \Cref{fig:heatmap_by_environment} shows how performance varies with the sampling budget: CoFlow reaches near-peak scores with few updates on most splits.

The heatmap shows that one-pass generation is sufficient for quality on most configurations. On the training side, \Cref{fig:training_compute} shows that the finite-difference implementation averages 1.78$\times$ faster training steps and 41.1\% lower peak GPU memory than the tested explicit JVP implementation across the seven displayed task shapes. These measurements quantify the per-update savings from approximating the consistency correction of the joint model under the same training setup, using synthetic batches and weighting task shapes equally. On the inference side, \Cref{fig:model_time_keyframes_main} compares accumulated model-only inference time on an NVIDIA A100. Each timestamp multiplies per-decision sampling latency by elapsed rollout steps, excluding simulator time. Because CoFlow uses one flow update per decision while MADiff uses 15 DDIM denoising steps, CoFlow is faster on the displayed rollouts; the complete timing table in \Cref{tab:model_inference_time} gives an average model-sampling speedup of 12.93$\times$.

\textbf{Tasks shown in the keyframes.} In \Cref{fig:model_time_keyframes_main}\textbf{(a)}, three agents (colored circles) must spread out to cover three landmarks (stars) while avoiding collisions; the trails show their movement toward different targets. In panel \textbf{(b)}, eight allied marines (blue) fight eight enemy marines (red). Each agent controls one allied unit, choosing movement and attack targets to defeat the opposing team; health bars indicate remaining health. The eight snapshots in each row progress from left to right. These policy rollouts show the behavior of the corresponding methods and are annotated with accumulated model-sampling times.

\textbf{Additional appendix material.} The appendix documents dataset return distributions in \Cref{fig:violin_dataset}, complete model-inference timing in \Cref{tab:model_inference_time}, the full reward heatmap and training ablations in \Cref{fig:heatmap_by_environment,tab:training_ablations}, gate diagnostics in \Cref{fig:gamma_signs}, threshold-based few-step behavior in \Cref{fig:steps_to_dataset_mean}, and rollout-level timing evidence in \Cref{fig:model_time_keyframes}.

\section{Conclusion}

\textbf{Conclusion.} We presented CoFlow, a distillation-free framework for coordinated few-step generation in offline multi-agent reinforcement learning. CoFlow learns a joint averaged-velocity field for few-step generation, uses CVA to model cross-agent dependencies, and employs finite-difference consistency training to reduce the training computation and memory costs of this joint model. Experiments on MPE and SMAC demonstrate strong task performance with few sampling steps, and ablations support the contributions of both components. These results highlight the value of embedding coordination directly into the generative model for efficient multi-agent decision making.

\label{sec:limitations}
\textbf{Limitations.} Several directions remain open. \textbf{(1)}~For tasks requiring additional refinement, adaptive-step generation is promising: a budget-aware criterion could decide when multiple steps are needed relative to a fixed schedule. \textbf{(2)}~Our experiments cover teams of up to eight agents; computational efficiency and coordination performance in larger teams remain to be evaluated. \textbf{(3)}~Generalization beyond CVA to communication-, graph-, or value-decomposition-based coordination architectures is an important next step. \textbf{(4)}~The performance gap between centralized and decentralized execution varies across tasks. Improving consistency among locally conditioned predictions is an important direction for future work.

\subsection*{AI use statement}
We used AI tools only to polish the language of the manuscript. All technical content, experiments, analyses, claims, and final wording were reviewed and approved by the authors.

\begin{ack}
\end{ack}

\bibliography{main}

@misc{levine2020offline,
  title={Offline reinforcement learning: Tutorial, review, and perspectives on open problems},
  author={Levine, Sergey and Kumar, Aviral and Tucker, George and Fu, Justin},
  howpublished={arXiv preprint arXiv:2005.01643},
  year={2020},
  eprint={2005.01643},
  archivePrefix={arXiv},
}

@inproceedings{kumar2020conservative,
  title={Conservative q-learning for offline reinforcement learning},
  author={Kumar, Aviral and Zhou, Aurick and Tucker, George and Levine, Sergey},
  booktitle={Advances in Neural Information Processing Systems (NeurIPS)},
  pages={1179--1191},
  year={2020}
}

@inproceedings{yang2021believe,
  title={Believe what you see: Implicit constraint approach for offline multi-agent reinforcement learning},
  author={Yang, Yiqin and Ma, Xiaoteng and Li, Chenghao and Zheng, Zewu and Zhang, Qiyuan and Huang, Gao and Yang, Jun and Zhao, Qianchuan},
  booktitle={Advances in Neural Information Processing Systems (NeurIPS)},
  pages={10299--10312},
  year={2021}
}

@inproceedings{pan2022plan,
  title={Plan better amid conservatism: Offline multi-agent reinforcement learning with actor rectification},
  author={Pan, Ling and Huang, Longbo and Ma, Tengyu and Xu, Huazhe},
  booktitle={International Conference on Machine Learning (ICML)},
  pages={17221--17237},
  year={2022}
}

@inproceedings{janner2022planning,
  title={Planning with diffusion for flexible behavior synthesis},
  author={Janner, Michael and Du, Yilun and Tenenbaum, Joshua B and Levine, Sergey},
  booktitle={International Conference on Machine Learning (ICML)},
  pages={9902--9915},
  year={2022}
}

@inproceedings{liu2023flow,
  title={Flow straight and fast: Learning to generate and transfer data with rectified flow},
  author={Liu, Xingchao and Gong, Chengyue and Liu, Qiang},
  booktitle={International Conference on Learning Representations (ICLR)},
  year={2023},
}

@inproceedings{mordatch2018emergence,
  title={Emergence of grounded compositional language in multi-agent populations},
  author={Mordatch, Igor and Abbeel, Pieter},
  booktitle={AAAI Conference on Artificial Intelligence (AAAI)},
  year={2018}
}

@inproceedings{samvelyan2019starcraft,
  title={The {StarCraft} Multi-Agent Challenge},
  author={Samvelyan, Mikayel and Rashid, Tabish and Schroeder de Witt, Christian and Farquhar, Gregory and Nardelli, Nantas and Rudner, Tim G. J. and Hung, Chia-Man and Torr, Philip H. S. and Foerster, Jakob and Whiteson, Shimon},
  booktitle={International Conference on Autonomous Agents and Multiagent Systems (AAMAS)},
  pages={2186--2188},
  year={2019},
}

@inproceedings{zhu2024madiff,
  title={{MADiff}: Offline multi-agent learning with diffusion models},
  author={Zhu, Zhengbang and Liu, Minghuan and Mao, Liyuan and Kang, Bingyi and Xu, Minkai and Yu, Yong and Ermon, Stefano and Zhang, Weinan},
  booktitle={Advances in Neural Information Processing Systems (NeurIPS)},
  pages={4177--4206},
  year={2024}
}

@article{meng2023madt,
  title={Offline Pre-trained Multi-agent Decision Transformer},
  author={Meng, Linghui and Wen, Muning and Le, Chenyang and Li, Xiyun and Xing, Dengpeng and Zhang, Weinan and Wen, Ying and Zhang, Haifeng and Wang, Jun and Yang, Yaodong and Xu, Bo},
  journal={Machine Intelligence Research},
  volume={20},
  pages={233--248},
  year={2023},
}

@article{hernandez2019survey,
  title={A survey and critique of multiagent deep reinforcement learning},
  author={Hernandez-Leal, Pablo and Kartal, Bilal and Taylor, Matthew E},
  journal={Autonomous Agents and Multi-Agent Systems},
  volume={33},
  number={6},
  pages={750--797},
  year={2019},
  publisher={Springer}
}

@inproceedings{lipman2023flow,
  title={Flow Matching for Generative Modeling},
  author={Lipman, Yaron and Chen, Ricky T. Q. and Ben-Hamu, Heli and Nickel, Maximilian and Le, Matthew},
  booktitle={International Conference on Learning Representations (ICLR)},
  year={2023},
  url={https://openreview.net/forum?id=PqvMRDCJT9t}
}

@inproceedings{chi2023diffusion,
  title={Diffusion policy: Visuomotor policy learning via action diffusion},
  author={Chi, Cheng and Feng, Siyuan and Du, Yilun and Xu, Zhenjia and Cousineau, Eric and Burchfiel, Benjamin and Song, Shuran},
  booktitle={Robotics: Science and Systems (RSS)},
  year={2023}
}

@inproceedings{wang2023diffusion,
  title={Diffusion policies as an expressive policy class for offline reinforcement learning},
  author={Wang, Zhendong and Hunt, Jonathan J and Zhou, Mingyuan},
  booktitle={International Conference on Learning Representations (ICLR)},
  year={2023}
}

@inproceedings{lee2025macflow,
  title={Multi-agent Coordination via Flow Matching},
  author={Lee, Dongsu and Lee, Daehee and Zhang, Amy},
  booktitle={International Conference on Learning Representations (ICLR)},
  year={2026}
}

@inproceedings{li2025dof,
  title={{DoF}: A Diffusion Factorization Framework for Offline Multi-Agent Reinforcement Learning},
  author={Li, Chao and Deng, Ziwei and Lin, Chenxing and Chen, Wenqi and Fu, Yongquan and Liu, Weiquan and Wen, Chenglu and Wang, Cheng and Shen, Siqi},
  booktitle={International Conference on Learning Representations (ICLR)},
  year={2025}
}

@inproceedings{zhang2025efm,
  title={Energy-Weighted Flow Matching for Offline Reinforcement Learning},
  author={Zhang, Shiyuan and Zhang, Weitong and Gu, Quanquan},
  booktitle={International Conference on Learning Representations (ICLR)},
  year={2025}
}

@inproceedings{barde2024offline,
  title={A Model-Based Solution to the Offline Multi-Agent Reinforcement Learning Coordination Problem},
  author={Barde, Paul and Foerster, Jakob N. and Nowrouzezahrai, Derek and Zhang, Amy},
  booktitle={Proceedings of the 23rd International Conference on Autonomous Agents and Multiagent Systems (AAMAS)},
  pages={141--150},
  year={2024},
  url={https://www.ifaamas.org/Proceedings/aamas2024/pdfs/p141.pdf}
}

@inproceedings{qiao2025omsd,
  title={{OMSD}: Offline Multi-agent Reinforcement Learning via Sequential Score Decomposition},
  author={Qiao, Dan and Li, Wenhao and Yang, Shanchao and Zha, Hongyuan and Wang, Baoxiang},
  booktitle={International Conference on Machine Learning (ICML)},
  year={2026},
}

@inproceedings{yuan2025madits,
  title={{MADiTS}: Efficient Multi-agent Offline Coordination via Diffusion-based Trajectory Stitching},
  author={Yuan, Lei and Bian, Yuqi and Li, Lihe and Zhang, Ziqian and Guan, Cong and Yu, Yang},
  booktitle={International Conference on Learning Representations (ICLR)},
  year={2025}
}

@inproceedings{zhan2025exploiting,
  title={Exploiting Structure in Offline Multi-Agent {RL}: The Benefits of Low Interaction Rank},
  author={Zhan, Wenhao and Fujimoto, Scott and Zhu, Zheqing and Lee, Jason D. and Jiang, Daniel R. and Efroni, Yonathan},
  booktitle={International Conference on Learning Representations (ICLR)},
  year={2025}
}

@inproceedings{liu2025inspo,
  title={{InSPO}: Offline Multi-Agent Reinforcement Learning via In-Sample Sequential Policy Optimization},
  author={Liu, Zongkai and Lin, Qian and Yu, Chao and Wu, Xiawei and Liang, Yile and Li, Donghui and Ding, Xuetao},
  booktitle={AAAI Conference on Artificial Intelligence (AAAI)},
  pages={19068--19076},
  year={2025},
}

@inproceedings{lu2025diffusion_veteran,
  title={What Makes a Good Diffusion Planner for Decision Making?},
  author={Lu, Haofei and Han, Dongqi and Shen, Yifei and Li, Dongsheng},
  booktitle={International Conference on Learning Representations (ICLR)},
  year={2025},
  note={Spotlight}
}

@inproceedings{mcallister2026fpo,
  title={Flow Matching Policy Gradients},
  author={McAllister, David and Ge, Songwei and Yi, Brent and Kim, Chung Min and Weber, Ethan and Choi, Hongsuk and Feng, Haiwen and Kanazawa, Angjoo},
  booktitle={International Conference on Learning Representations (ICLR)},
  year={2026}
}

@inproceedings{zhang2025reinflow,
  title={{ReinFlow}: Fine-tuning Flow Matching Policy with Online Reinforcement Learning},
  author={Zhang, Tonghe and Yu, Chao and Su, Sichang and Wang, Yu},
  booktitle={Advances in Neural Information Processing Systems (NeurIPS)},
  year={2025}
}

@inproceedings{park2025flowql,
  title={Flow {Q}-Learning},
  author={Park, Seohong and Li, Qiyang and Levine, Sergey},
  booktitle={International Conference on Machine Learning (ICML)},
  pages={48104--48127},
  year={2025},
}

@inproceedings{frans2025shortcut,
  title={One Step Diffusion via Shortcut Models},
  author={Frans, Kevin and Hafner, Danijar and Levine, Sergey and Abbeel, Pieter},
  booktitle={International Conference on Learning Representations (ICLR)},
  year={2025}
}

@inproceedings{formanek2023otg,
  title={Off-the-Grid {MARL}: Datasets and Baselines for Offline Multi-Agent Reinforcement Learning},
  author={Formanek, Claude and Jeewa, Asad and Shock, Jonathan and Pretorius, Arnu},
  booktitle={International Conference on Autonomous Agents and Multiagent Systems (AAMAS)},
  pages={2442--2444},
  year={2023},
}

@inproceedings{formanek2024dispelling,
  title={Dispelling the Mirage of Progress in Offline {MARL} through Standardised Baselines and Evaluation},
  author={Formanek, Claude and Tilbury, Callum R. and Beyers, Louise and Shock, Jonathan and Pretorius, Arnu},
  booktitle={Advances in Neural Information Processing Systems (NeurIPS), Datasets and Benchmarks Track},
  year={2024}
}

@inproceedings{fujimoto2021minimalist,
  title={A Minimalist Approach to Offline Reinforcement Learning},
  author={Fujimoto, Scott and Gu, Shixiang Shane},
  booktitle={Advances in Neural Information Processing Systems (NeurIPS)},
  pages={20132--20145},
  year={2021}
}

@inproceedings{kostrikov2022iql,
  title={Offline Reinforcement Learning with Implicit {Q}-Learning},
  author={Kostrikov, Ilya and Nair, Ashvin and Levine, Sergey},
  booktitle={International Conference on Learning Representations (ICLR)},
  year={2022}
}

@inproceedings{wang2023omiga,
  title={Offline Multi-Agent Reinforcement Learning with Implicit Global-to-Local Value Regularization},
  author={Wang, Xiangsen and Xu, Haoran and Zheng, Yinan and Zhan, Xianyuan},
  booktitle={Advances in Neural Information Processing Systems (NeurIPS)},
  year={2023},
}

@inproceedings{ho2020ddpm,
  title={Denoising Diffusion Probabilistic Models},
  author={Ho, Jonathan and Jain, Ajay and Abbeel, Pieter},
  booktitle={Advances in Neural Information Processing Systems (NeurIPS)},
  pages={6840--6851},
  year={2020}
}

@misc{dmpo,
  title={One Step Is Enough: Dispersive {MeanFlow} Policy Optimization},
  author={Zou, Guowei and Wang, Haitao and Wu, Hejun and Qian, Yukun and Wang, Yuhang and Li, Weibing},
  howpublished={arXiv preprint arXiv:2601.20701},
  year={2026},
  eprint={2601.20701},
  archivePrefix={arXiv},
  primaryClass={cs.RO},
}

@inproceedings{d2ppo,
  title={{D$^2$PPO}: Diffusion Policy Policy Optimization with Dispersive Loss},
  author={Zou, Guowei and Li, Weibing and Wu, Hejun and Qian, Yukun and Wang, Yuhang and Wang, Haitao},
  booktitle={AAAI Conference on Artificial Intelligence (AAAI)},
  pages={18891--18899},
  year={2026},
}

@inproceedings{song2023consistency,
  title={Consistency models},
  author={Song, Yang and Dhariwal, Prafulla and Chen, Mark and Sutskever, Ilya},
  booktitle={International Conference on Machine Learning (ICML)},
  pages={32211--32252},
  year={2023}
}

@inproceedings{salimans2022progressive,
  title={Progressive distillation for fast sampling of diffusion models},
  author={Salimans, Tim and Ho, Jonathan},
  booktitle={International Conference on Learning Representations (ICLR)},
  year={2022}
}

@inproceedings{song2020ddim,
  title={Denoising diffusion implicit models},
  author={Song, Jiaming and Meng, Chenlin and Ermon, Stefano},
  booktitle={International Conference on Learning Representations (ICLR)},
  year={2021}
}

@inproceedings{ren2024dppo,
  title={Diffusion Policy Policy Optimization},
  author={Ren, Allen Z. and Lidard, Justin and Ankile, Lars L. and Simeonov, Anthony and Agrawal, Pulkit and Majumdar, Anirudha and Burchfiel, Benjamin and Dai, Hongkai and Simchowitz, Max},
  booktitle={International Conference on Learning Representations (ICLR)},
  year={2025}
}

@misc{hansen2023idql,
  title={{IDQL}: Implicit {Q}-learning as an actor-critic method with diffusion policies},
  author={Hansen-Estruch, Philippe and Kostrikov, Ilya and Janner, Michael and Kuba, Jakub Grudzien and Levine, Sergey},
  howpublished={arXiv preprint arXiv:2304.10573},
  year={2023},
  eprint={2304.10573},
  archivePrefix={arXiv},
}

@inproceedings{chen2023sfbc,
  title={Offline reinforcement learning via high-fidelity generative behavior modeling},
  author={Chen, Huayu and Lu, Cheng and Ying, Chengyang and Su, Hang and Zhu, Jun},
  booktitle={International Conference on Learning Representations (ICLR)},
  year={2023}
}

@inproceedings{gat2024discrete,
  title={Discrete flow matching},
  author={Gat, Itai and Remez, Tal and Shaul, Neta and Kreuk, Felix and Chen, Ricky T. Q. and Synnaeve, Gabriel and Adi, Yossi and Lipman, Yaron},
  booktitle={Advances in Neural Information Processing Systems (NeurIPS)},
  year={2024}
}

@inproceedings{dao2024latent,
  title={Flow matching in latent space},
  author={Dao, Quan and Phung, Hao and Nguyen, Binh and Tran, Anh},
  booktitle={International Conference on Learning Representations (ICLR)},
  year={2024}
}

@inproceedings{lu2024contrastive,
  title={Contrastive Energy Prediction for Exact Energy-Guided Diffusion Sampling in Offline Reinforcement Learning},
  author={Lu, Cheng and Chen, Huayu and Chen, Jianfei and Su, Hang and Li, Chongxuan and Zhu, Jun},
  booktitle={Proceedings of the 40th International Conference on Machine Learning (ICML)},
  series={Proceedings of Machine Learning Research},
  volume={202},
  pages={22825--22855},
  year={2023},
  url={https://proceedings.mlr.press/v202/lu23d.html}
}

@inproceedings{song2021score,
  title={Score-based generative modeling through stochastic differential equations},
  author={Song, Yang and Sohl-Dickstein, Jascha and Kingma, Diederik P and Kumar, Abhishek and Ermon, Stefano and Poole, Ben},
  booktitle={International Conference on Learning Representations (ICLR)},
  year={2021}
}

@inproceedings{rombach2022latent,
  title={High-resolution image synthesis with latent diffusion models},
  author={Rombach, Robin and Blattmann, Andreas and Lorenz, Dominik and Esser, Patrick and Ommer, Bj{\"o}rn},
  booktitle={IEEE/CVF Conference on Computer Vision and Pattern Recognition (CVPR)},
  pages={10684--10695},
  year={2022}
}

@inproceedings{geng2025avgvelocity,
  title={Mean Flows for One-step Generative Modeling},
  author={Geng, Zhengyang and Deng, Mingyang and Bai, Xingjian and Kolter, J. Zico and He, Kaiming},
  booktitle={Advances in Neural Information Processing Systems (NeurIPS)},
  year={2025},
}

@inproceedings{geng2025fastforward,
  title={Improved Mean Flows: On the Challenges of Fastforward Generative Models},
  author={Geng, Zhengyang and Lu, Yiyang and Wu, Zongze and Shechtman, Eli and Kolter, J. Zico and He, Kaiming},
  booktitle={Proceedings of the IEEE/CVF Conference on Computer Vision and Pattern Recognition (CVPR)},
  pages={30467--30476},
  year={2026},
  url={https://openaccess.thecvf.com/content/CVPR2026/html/Geng_Improved_Mean_Flows_On_the_Challenges_of_Fastforward_Generative_Models_CVPR_2026_paper.html}
}

@misc{guo2025intervalsplitting,
  title={{SplitMeanFlow}: Interval Splitting Consistency in Few-Step Generative Modeling},
  author={Guo, Yi and Wang, Wei and Yuan, Zhihang and Cao, Rong and Chen, Kuan and Chen, Zhengyang and Huo, Yuanyuan and Zhang, Yang and Wang, Yuping and Liu, Shouda and Wang, Yuxuan},
  howpublished={arXiv preprint arXiv:2507.16884},
  year={2025},
  eprint={2507.16884},
  archivePrefix={arXiv},
  primaryClass={cs.LG},
}

@inproceedings{pomerleau1988alvinn,
  title={{ALVINN}: An autonomous land vehicle in a neural network},
  author={Pomerleau, Dean A.},
  booktitle={Advances in Neural Information Processing Systems (NeurIPS)},
  year={1988}
}
\bibliographystyle{plainnat}

\newpage
\appendix
\makeatletter
\setlength{\@fptop}{0pt}
\setlength{\@fpsep}{10pt}
\setlength{\@fpbot}{0pt plus 1fil}
\makeatother

\begin{center}
{\Large \textbf{Appendix}}
\end{center}

\vspace{0.5em}

\section*{Appendix Contents}

\begin{center}
\begin{tabular}{@{}p{0.14\textwidth}p{0.78\textwidth}@{}}
\toprule
\textbf{Section} & \textbf{Title} \\
\midrule
Appendix A & Experimental Details and Results \\
A.1 & Experimental Setup \\
A.2 & Benchmarks and Offline Datasets \\
A.3 & Full Performance Summary \\
A.4 & Ablations \\
A.5 & Additional Coordination Analysis \\
\midrule
Appendix B & Method Details and Theoretical Foundations \\
B.1 & Preliminaries \\
B.2 & Notation \\
B.3 & Implementation Notes \\
B.4 & Training and Inference Algorithms \\
B.5 & Taylor Control of the Finite-Difference Surrogate \\
B.6 & Joint Velocity Decomposition \\
B.7 & Single-Pass Wasserstein Error Bound \\
B.8 & Scaling with the Number of Agents \\
\bottomrule
\end{tabular}
\end{center}

\section{Experimental Details and Results}
\label{app:experiments}

This section reports the experimental protocol, benchmark datasets, full results, ablations, few-step analyses, and coordination diagnostics.

\subsection{Experimental Setup}

\textbf{Network architecture.} CoFlow uses a temporal U-Net with encoder-decoder skip connections, 1D convolutional blocks with kernel size 5, group normalization with 8 groups, Mish activations, channel multipliers [1, 4, 8], and base dimension 128. The multi-scale skip structure provides natural insertion points for gated CVA while keeping few-step trajectory generation memory-efficient.

\textbf{Coordinated Velocity Attention (CVA).} We apply 4-head CVA at the bottleneck and before each decoder skip fusion. Given skip features $h^i_l \in \mathbb{R}^{T_l \times F_l}$ for agent $i$ at resolution $l$, CVA attends across the agent dimension and adds its output back through the zero-initialized residual gate in \Cref{eq:cva}; the decoder then concatenates the enhanced skip with its current state.

\textbf{Time conditioning and action decoding.} The averaged velocity model conditions on the flow endpoint time $t$ and reference time $r$, equivalently on $t$ and $t-r$, using sinusoidal encodings injected by feature-wise linear modulation (FiLM). For action extraction, a separate 3-layer inverse-dynamics MLP with hidden dimension 256 maps consecutive observation pairs $(o_q, o_{q+1})$ to $a_q$ and stays outside the velocity U-Net.

\subsection{Benchmarks and Offline Datasets}

\textbf{Training-compute benchmark.} For \Cref{fig:training_compute}, both implementations use the same shared CVA network, synthetic Gaussian batches, AdamW with learning rate $2\!\times\!10^{-4}$, seed 100, guidance weight 1.2, and finite-difference step $10^{-3}$. Each timed update includes the velocity loss, backward pass, and optimizer step; data loading, inverse dynamics, and environment simulation are excluded. Time is averaged over 20 updates after 5 warm-up updates, and memory is the maximum allocated GPU memory over those updates. We average per-task JVP/FD time ratios and percentage memory reductions. The finite-difference implementation follows the velocity-model update in \Cref{alg:training}, using the sample-conditional tangent $\mathbf v_c=\mathbf z_1-\boldsymbol\tau_0$ and the two-time-conditioned joint field $\mathbf u_\theta(\mathbf z_t,r,t)$. The JVP implementation uses the same training setup and replaces the finite-difference correction in \Cref{eq:fd_correction} with the explicit directional derivative in \Cref{eq:jvp_consistency}; both corrections are detached during optimization. Measurements were collected on a shared A100 server.

\textbf{Baseline-compatible environment protocol.} We follow the released offline datasets, reward definitions, and benchmark configurations used by prior offline MARL baselines. MPE and SMAC use the original tasks with matched baseline results.

\textbf{MPE score normalization.} Following MADiff and other MPE baselines, we report OMAR-normalized score $100\, (R-R_{\mathrm{random}})/(R_{\mathrm{expert}}-R_{\mathrm{random}})$, where $R$ is the evaluated return and $R_{\mathrm{random}},R_{\mathrm{expert}}$ are the fixed random-policy and expert-return anchors. We apply the same transformation to offline-data means and CoFlow returns.

\textbf{Datasets.} We evaluate on MPE and SMAC. MPE includes Spread, Tag, and World with expert, medium-replay, medium, and random datasets. SMAC covers symmetric 3m and 8m battles, heterogeneous 2s3z, and asymmetric 5m\_vs\_6m micromanagement tasks.

\begin{center}
\centering
\includegraphics[width=0.80\textwidth]{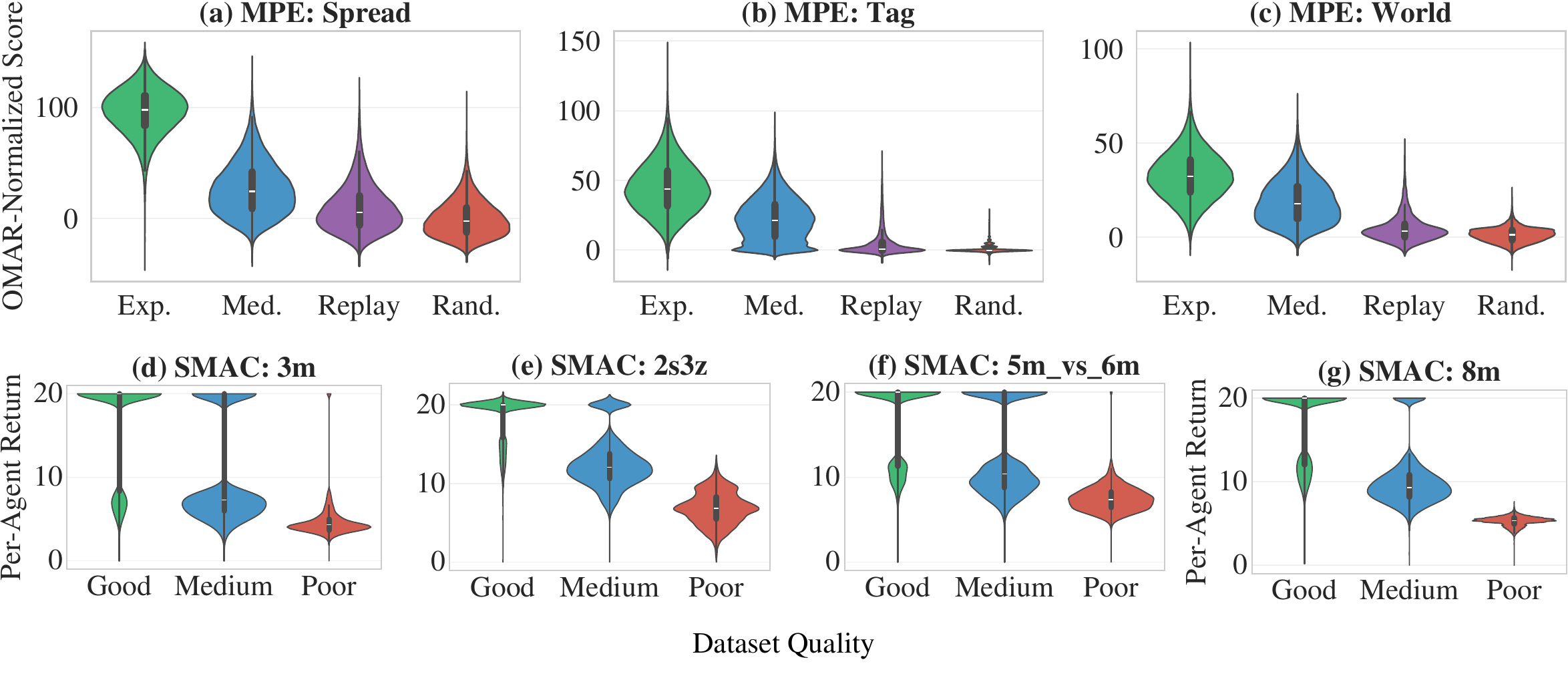}
\captionof{figure}{Offline dataset distributions. MPE values are OMAR-normalized scores computed from per-agent returns. SMAC values are per-agent returns.}
\label{fig:violin_dataset}
\end{center}

\noindent\textbf{Dataset return distributions.} \Cref{fig:violin_dataset} shows violin plots of the offline dataset returns. In each subplot, the horizontal axis indexes dataset quality and the vertical axis gives the corresponding return scale; MPE uses OMAR-normalized score, while SMAC uses per-agent returns. The violin width indicates the density of trajectories at each return level. The plots verify that the benchmark covers both high-return expert data and lower-return medium or random data, so the main comparison tests CoFlow across substantially different offline-data regimes.

\begin{center}
\begin{minipage}{\textwidth}
\captionof{table}{Relationship between sampling method and model-inference cost across benchmark task and dataset-quality pairs, measured on an NVIDIA A100. CoFlow uses one flow step and MADiff uses 15 DDIM denoising steps; CoF and MAD denote CoFlow and MADiff.}
\label{tab:model_inference_time}
\centering
\scriptsize
\setlength{\tabcolsep}{2.0pt}
\renewcommand{\arraystretch}{0.98}
\resizebox{\textwidth}{!}{%
\begin{tabular}{clrrrrrr@{\hspace{7pt}}clrrrrrr}
\toprule
\multicolumn{8}{c}{\textbf{MPE}} &
\multicolumn{8}{c}{\textbf{SMAC}} \\
\midrule
Task & Quality & Steps & \multicolumn{2}{c}{Latency (ms)} & \multicolumn{2}{c}{Rollout (s)} & Speedup &
Task & Quality & Steps & \multicolumn{2}{c}{Latency (ms)} & \multicolumn{2}{c}{Rollout (s)} & Speedup \\
\cmidrule(lr){4-5}\cmidrule(lr){6-7}\cmidrule(lr){12-13}\cmidrule(lr){14-15}
 & & & CoF & MAD & CoF & MAD & MAD/CoF &
 & & & CoF & MAD & CoF & MAD & MAD/CoF \\
\midrule
\multirow[c]{4}{*}{Spread} & Expert & 25 & 19.8 & 270.7 & 0.495 & 6.768 & 13.68$\times$ &
\multirow[c]{3}{*}{3m} & Poor & 60 & 20.2 & 275.1 & 1.215 & 16.509 & 13.59$\times$ \\
 & Med.-Rep. & 25 & 19.2 & 273.8 & 0.481 & 6.844 & 14.23$\times$ &
 & Medium & 58 & 26.5 & 275.1 & 1.540 & 15.958 & 10.37$\times$ \\
 & Medium & 25 & 20.2 & 264.2 & 0.506 & 6.606 & 13.06$\times$ &
 & Good & 35 & 19.7 & 277.9 & 0.689 & 9.726 & 14.11$\times$ \\
\cmidrule(lr){9-16}
 & Random & 25 & 20.3 & 264.8 & 0.508 & 6.619 & 13.03$\times$ &
\multirow[c]{3}{*}{2s3z} & Poor & 120 & 21.2 & 281.6 & 2.546 & 33.795 & 13.27$\times$ \\
\cmidrule(lr){1-8}
\multirow[c]{4}{*}{Tag} & Expert & 25 & 19.4 & 262.9 & 0.486 & 6.573 & 13.53$\times$ &
 & Medium & 98 & 19.9 & 273.3 & 1.953 & 26.784 & 13.71$\times$ \\
 & Med.-Rep. & 25 & 19.6 & 259.1 & 0.491 & 6.478 & 13.19$\times$ &
 & Good & 116 & 21.1 & 270.1 & 2.449 & 31.328 & 12.79$\times$ \\
\cmidrule(lr){9-16}
 & Medium & 25 & 19.4 & 262.3 & 0.485 & 6.558 & 13.51$\times$ &
\multirow[c]{3}{*}{\shortstack{5m\\vs.\\6m}} & Poor & 70 & 22.7 & 284.7 & 1.588 & 19.926 & 12.55$\times$ \\
 & Random & 25 & 19.0 & 262.6 & 0.476 & 6.564 & 13.80$\times$ &
 & Medium & 39 & 24.5 & 282.0 & 0.956 & 10.999 & 11.50$\times$ \\
\cmidrule(lr){1-8}
\multirow[c]{4}{*}{World} & Expert & 25 & 19.0 & 262.0 & 0.475 & 6.549 & 13.79$\times$ &
 & Good & 42 & 26.8 & 270.8 & 1.124 & 11.374 & 10.12$\times$ \\
\cmidrule(lr){9-16}
 & Med.-Rep. & 25 & 19.2 & 260.8 & 0.481 & 6.519 & 13.56$\times$ &
\multirow[c]{3}{*}{8m} & Poor & 120 & 24.5 & 301.4 & 2.939 & 36.170 & 12.31$\times$ \\
 & Medium & 25 & 18.9 & 262.0 & 0.472 & 6.550 & 13.88$\times$ &
 & Medium & 116 & 29.7 & 289.0 & 3.449 & 33.521 & 9.72$\times$ \\
 & Random & 25 & 19.7 & 274.8 & 0.492 & 6.869 & 13.95$\times$ &
 & Good & 114 & 22.2 & 290.6 & 2.536 & 33.128 & 13.07$\times$ \\
\bottomrule
\end{tabular}%
}
\vspace{2pt}
{\scriptsize\raggedright\textit{Note.} MADiff timing uses randomly initialized same-shape networks because timing depends mainly on the sampling graph and tensor shapes. These measurements quantify model-sampling cost.\par}
\end{minipage}
\end{center}

{\clubpenalty=10000 \widowpenalty=10000
\noindent\textbf{Complete 24-task model-inference timing.} \Cref{tab:model_inference_time} reports the corresponding model-only timing comparison for all 24 benchmark task and dataset-quality pairs used in \Cref{tab:main_results}. Rows enumerate suite, task, and dataset quality; columns report rollout length, per-call latency, full-rollout model time, and MADiff/CoFlow speedup. The table uses the same full-rollout convention as \Cref{fig:model_time_keyframes}: per-decision model sampling latency multiplied by the actual rollout length. CoFlow uses single-step sampling, while MADiff uses 15-step DDIM sampling. The speedup across all task and dataset-quality pairs demonstrates the model-sampling efficiency of one-step generation.\par}

\subsection{Full Performance Summary}

\begin{center}
\centering
\includegraphics[width=0.95\textwidth]{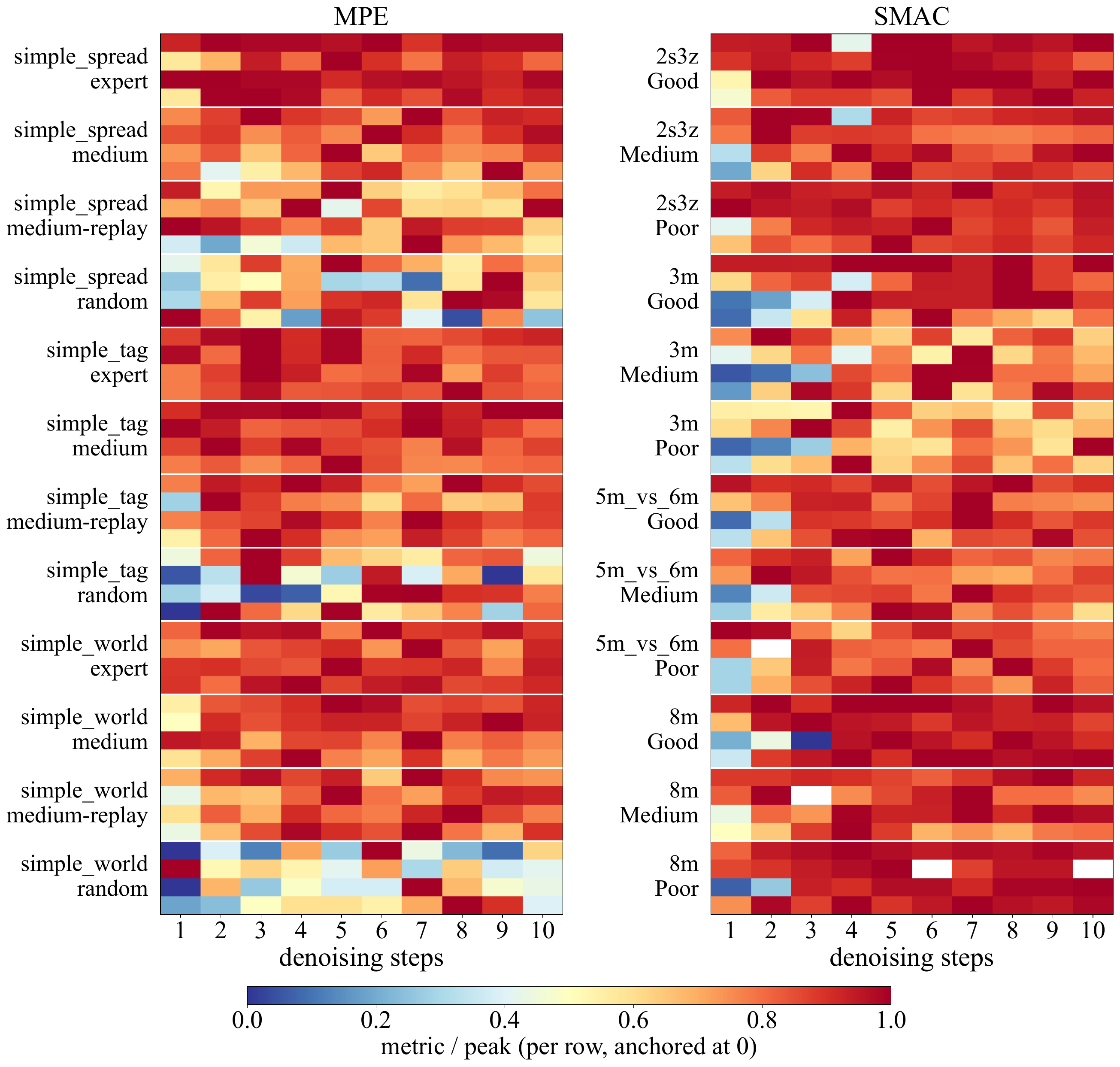}
\captionof{figure}{Environment-wise views of the denoising-step heatmap. Within each block for a task and dataset quality, rows follow CoFlow-C, CoFlow-D, CoFlow-base-C, and CoFlow-base-D.}
\label{fig:heatmap_by_environment}
\end{center}

{\clubpenalty=10000 \widowpenalty=10000
\noindent\textbf{Environment-wise summary of all 96 configurations.} \Cref{fig:heatmap_by_environment} compresses every performance value from the main experiments into two environment-wise heatmaps. The horizontal axis is the denoising-step budget, and each block for a task and dataset quality stacks CoFlow-C, CoFlow-D, CoFlow-base-C, and CoFlow-base-D. Cell color encodes performance as a fraction of the row's peak, so red indicates near-best performance for that configuration. CoFlow rows are usually near their row peak from step one, while CoFlow-base rows more often improve with additional denoising. This pattern supports the main few-step claim and shows that the remaining exceptions are localized to low-coverage Poor splits and asymmetric SMAC cases.\par}

\subsection{Ablations}

\textbf{Effect of guidance weight.} The classifier-free guidance weight $\omega$ controls the strength of return conditioning during inference. \Cref{tab:training_ablations} presents the effect of varying $\omega$ on MPE-Spread with expert data. Setting $\omega=1.0$ corresponds to standard conditional generation, while $\omega=1.2$ gives the best performance and raises the OMAR-normalized score from $97.2$ to $114.6$ (18\%). Larger values of 1.5 or 2.0 lower the score. We therefore use $\omega=1.2$ as the default across all experiments.

\begin{center}
\captionof{table}{Effect of guidance weight on MPE-Spread Expert.}
\label{tab:training_ablations}
\centering
\small
\renewcommand{\arraystretch}{1.1}
\begin{tabular}{@{}cc@{}}
\toprule
Guidance Weight $\omega$ & OMAR-Normalized Score \\
\midrule
1.0 & $97.2 \pm 4.1$ \\
1.2 & $\mathbf{114.6 \pm 0.5}$ \\
1.5 & $111.3 \pm 2.3$ \\
2.0 & $104.8 \pm 3.3$ \\
\bottomrule
\end{tabular}

\vspace{2pt}
{\footnotesize The best result is in \textbf{bold}. Results are reported as mean $\pm$ standard deviation across three random seeds.}
\end{center}

\subsection{Additional Coordination Analysis}

\begin{center}
\centering
\includegraphics[width=\textwidth]{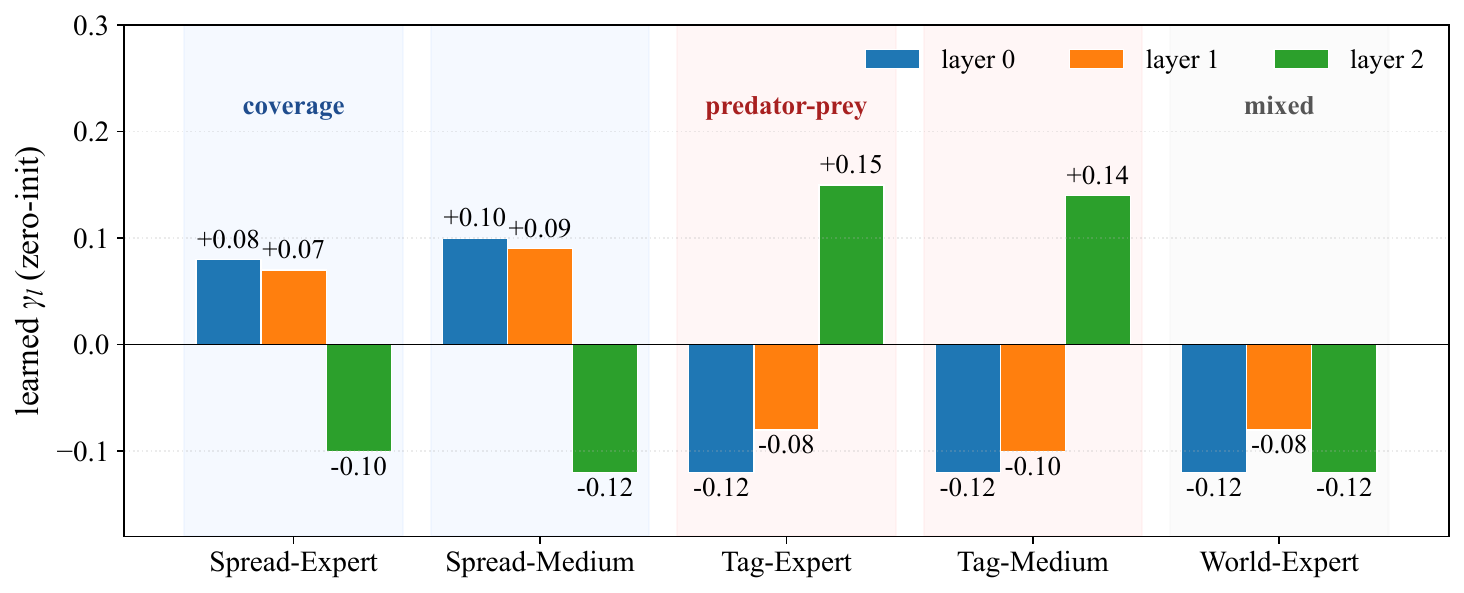}
\captionof{figure}{Relationship between MPE task type and learned CVA gate values across network depth. The plot reports $\gamma_l$ for representative CoFlow-C configurations.}
\label{fig:gamma_signs}
\end{center}

\noindent\textbf{Learned coordination structure.} \Cref{fig:gamma_signs} reports the learned CVA gate $\gamma_l$ across representative MPE tasks: the horizontal axis indexes task settings, the vertical axis gives the learned gate value, and colors distinguish CVA layers. The sign and magnitude indicate how strongly each layer injects the CVA correction into the skip pathway. \Cref{fig:attention_weights} in the main text complements this gate-level view with agent-to-agent attention matrices, where rows and columns correspond to agents and warmer cells indicate larger feature weights. Together, the two figures characterize task-dependent feature sharing through the attention weights and the strength of each layer's gated correction.

\begin{figure}[!p]
\centering
\includegraphics[width=\textwidth]{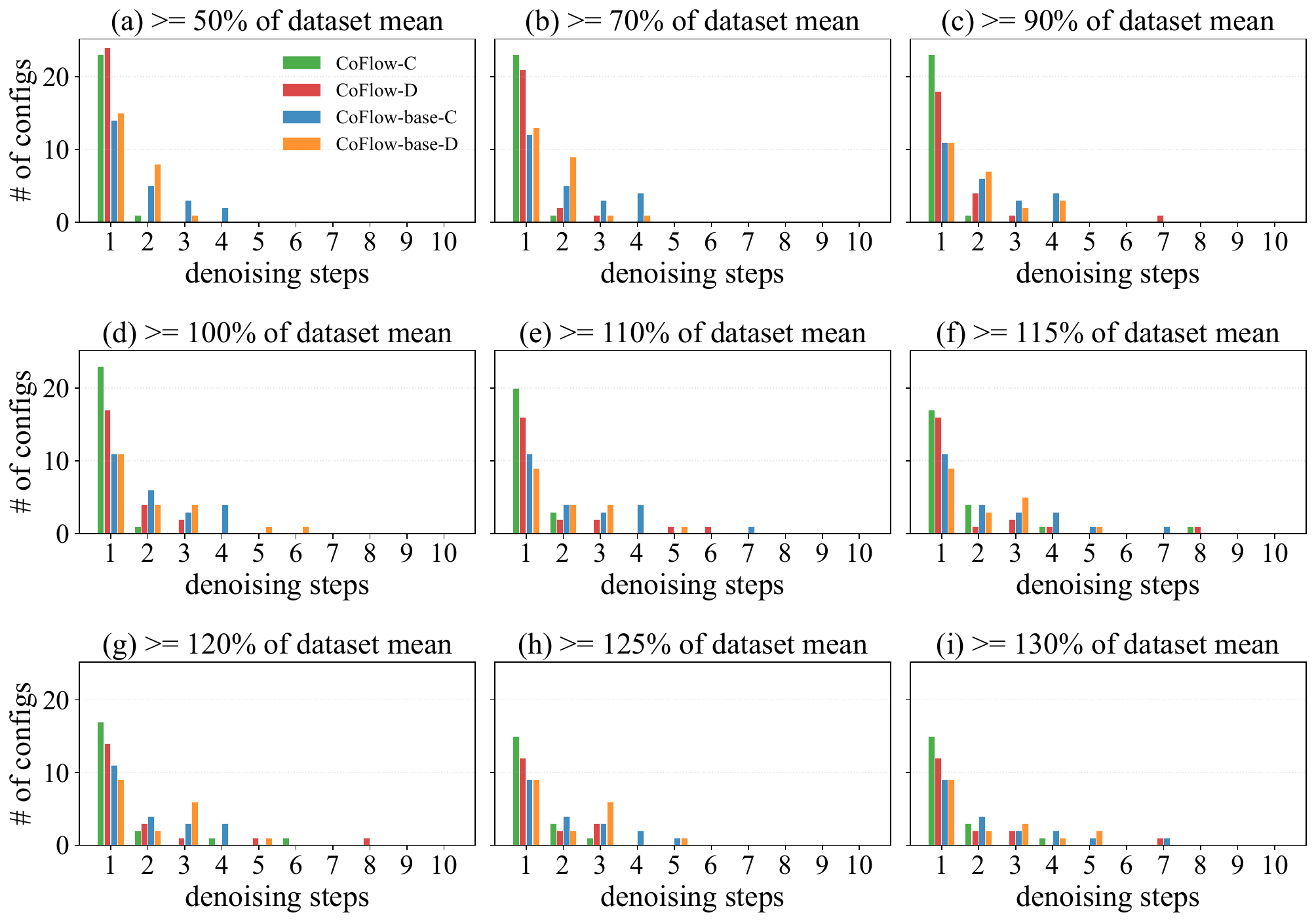}
\caption{Grouped view of denoising steps needed to reach target performance thresholds. Each subplot compares the four CoFlow variants at every denoising-step budget.}
\label{fig:steps_to_dataset_mean}
\par\vspace{4pt}

\includegraphics[width=\textwidth]{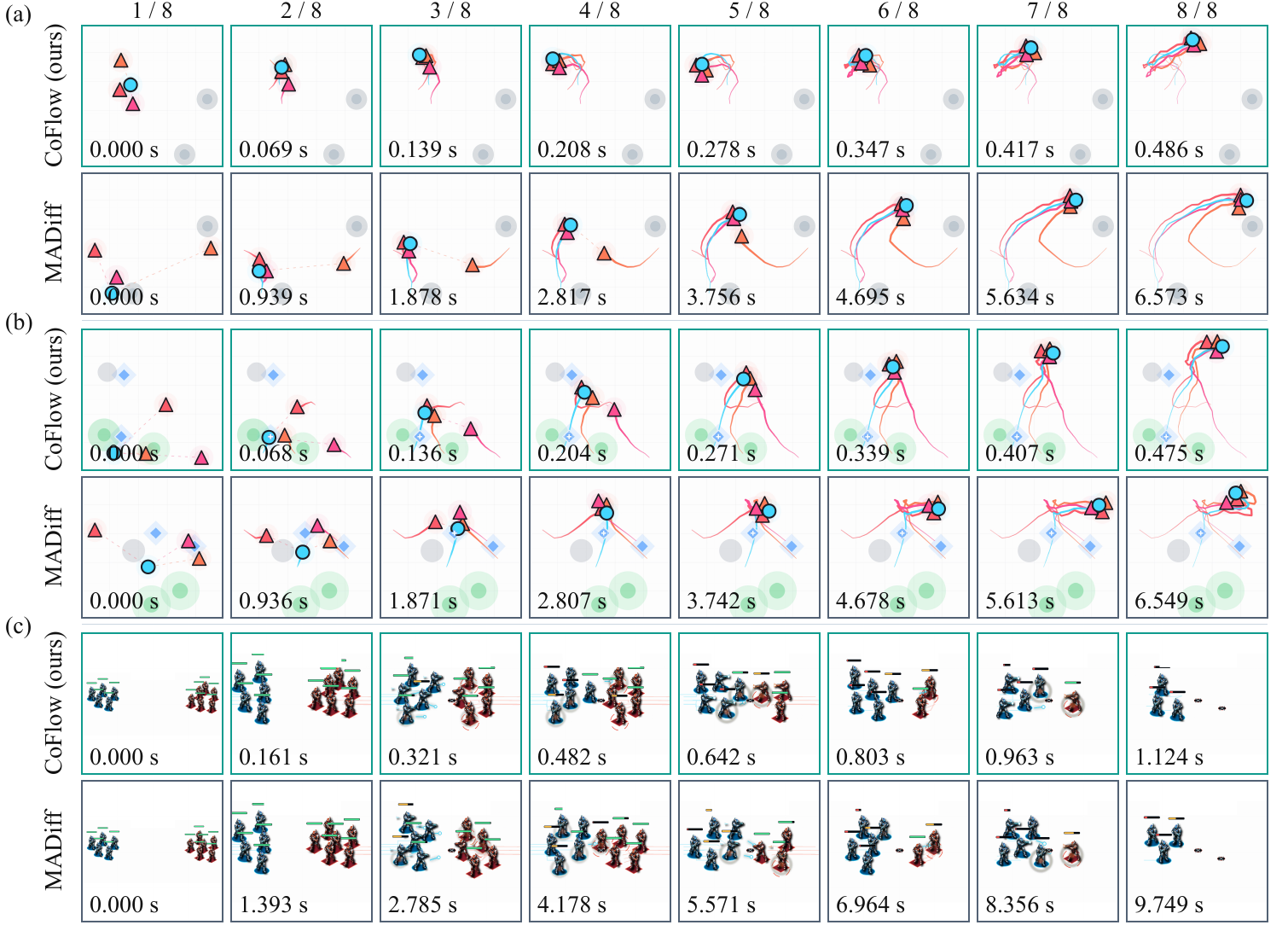}
\caption{Additional policy rollouts with accumulated model-sampling time: \textbf{(a)}~MPE-Tag, \textbf{(b)}~MPE-World, and \textbf{(c)}~SMAC-5m\_vs\_6m. Each adjacent row pair shows CoFlow and MADiff rollouts, respectively, with their corresponding model-sampling times.}
\label{fig:model_time_keyframes}
\end{figure}

\noindent\textbf{Steps to reach the offline dataset mean.} \Cref{fig:steps_to_dataset_mean} provides an alternative threshold-based view of few-step behavior. Each panel fixes a dataset-mean percentage threshold $\eta$; the horizontal axis is the first denoising step at which a configuration reaches that threshold, and the vertical axis counts how many configurations reach it at that step. Bars are grouped by the four CoFlow variants. The concentration of mass at the first few steps shows that most configurations reach dataset-level or above-dataset performance quickly, while the small right tails identify the settings where additional denoising remains useful.

\noindent\textbf{Tasks and visual symbols.} In \Cref{fig:model_time_keyframes}\textbf{(a)}, three predators (triangles) cooperate to catch a faster prey (cyan circle) while navigating around two fixed obstacles (gray disks). The team is rewarded for contacting the prey, so the converging trails illustrate pursuit. Panel \textbf{(b)} uses the same three-predator, one-prey objective in World, with one obstacle, two food locations (blue diamonds), and two forest regions (green disks). The prey can earn additional reward by reaching food, while the predators coordinate their pursuit. In panel \textbf{(c)}, five allied marines (blue) face six enemy marines (red); agents must coordinate movement and target selection to defeat a numerically larger team. Health bars and unit disappearances indicate damage and eliminations.

\noindent\textbf{Illustrative model-time keyframes.} Each task in \Cref{fig:model_time_keyframes} appears as a pair of adjacent rows for CoFlow and MADiff. To keep the MPE shapes and SMAC unit sprites legible at paper scale, each row shows eight uniformly spaced frames from a rollout of the corresponding policy. The visual sequences show policy behavior in the environment. The overlaid timestamps report the corresponding method's accumulated measured model-inference time. These totals multiply the measured per-decision sampling latency by the actual rollout length, separating model speed from simulator and rendering overhead.

\FloatBarrier
\section{Method Details and Theoretical Foundations}
\label{app:method_details}
\label{app:proofs}

This appendix gives background and notation, training and inference algorithms, and proofs for the finite-difference approximation, joint velocity decomposition, and one-step error bounds.

\subsubsection*{Part I: Background and Notation}

\subsection{Preliminaries}
\label{app:preliminaries}

We distinguish the conditional paths used to construct training pairs from the marginal flow whose interval displacement defines the averaged velocity. Ordinary flow-matching regression is background and defines the CoFlow-base ablation; the full CoFlow objective is \Cref{eq:velocity_loss}.

\smallskip\noindent\textbf{Ordinary conditional flow matching.} Sample $\boldsymbol\tau_0$ from the offline trajectory law $p_\tau$ and independent noise $\mathbf z_1\sim\mathcal N(0,\mathbf I)$, fixing the observation/return context. The linear interpolation and its sample-conditional velocity are given in \Cref{eq:training_pair}. Ordinary flow matching trains an instantaneous-velocity predictor by
\begin{equation}
\label{eq:flow_loss}
\mathcal L_{\mathrm{flow}}(\theta)
=\frac{1}{Nd}\mathbb E_{t,\boldsymbol\tau_0,\mathbf z_1}
\left\|\mathbf v_\theta(\mathbf z_t,t)-\mathbf v_c\right\|_2^2,
\end{equation}
where $\mathbf v_\theta$ predicts instantaneous velocity.

The population regression target is
\begin{equation}
\label{eq:marginal_velocity}
\mathbf v(\mathbf z,t)
=\mathbb E[\mathbf v_c\mid\mathbf z_t=\mathbf z].
\end{equation}

These conditional and marginal regression objectives have the same parameter-dependent part \citep{lipman2023flow,liu2023flow}. At the pure-noise endpoint, independence gives $\mathbf v(\mathbf z,1)=\mathbf z-\mathbb E[\boldsymbol\tau_0]$. Consequently, an exact single Euler update using this instantaneous endpoint velocity returns the conditional data mean, not a general sample from the target distribution.

\smallskip\noindent\textbf{Marginal transport and averaged velocity.} Let $\Phi_{s\leftarrow t}(\mathbf z)$ solve the marginal ODE
\begin{equation}
\frac{d}{ds}\Phi_{s\leftarrow t}(\mathbf z)
=\mathbf v(\Phi_{s\leftarrow t}(\mathbf z),s),
\qquad \Phi_{t\leftarrow t}(\mathbf z)=\mathbf z.
\end{equation}

Under the usual well-posedness and transport assumptions, $\Phi_{0\leftarrow1}$ maps the noise law to $p_\tau$. Individual conditional interpolation paths are straight, but the trajectories of this marginal ODE need not be straight. The ideal averaged velocity is therefore \citep{geng2025avgvelocity}
\begin{equation}
\label{eq:average_velocity_integral}
\mathbf u(\mathbf z,r,t)
=\frac{1}{t-r}\int_r^t
\mathbf v(\Phi_{s\leftarrow t}(\mathbf z),s)\,ds
=\frac{\mathbf z-\Phi_{r\leftarrow t}(\mathbf z)}{t-r}.
\end{equation}
This is the integral form of \Cref{eq:average_velocity}. It follows the marginal ODE trajectory, rather than the independently paired straight training interpolant. At $r=t$, $\mathbf u$ is defined by its continuous extension $\mathbf v(\mathbf z,t)$. In particular,
$\mathbf z_1-\mathbf u(\mathbf z_1,0,1)=\Phi_{0\leftarrow1}(\mathbf z_1)$.
The consistency identity below enables learning this interval displacement using conditional-velocity supervision, without computing the transport integral for each training example.

\smallskip\noindent\textbf{Wasserstein-2 distance.} For two probability measures $\mu,\nu$ on $\mathbb R^{Nd}$ with finite second moments,
\begin{equation}
W_2(\mu,\nu)
=\left(\inf_{\pi\in\Pi(\mu,\nu)}
\int\|x-y\|^2\,d\pi(x,y)\right)^{1/2},
\end{equation}
where $\Pi(\mu,\nu)$ contains the couplings of $\mu$ and $\nu$.

A shared-randomness coupling gives
\begin{equation}
W_2(\mu,\nu)\le
\left(\mathbb E_\xi\|F(\xi)-G(\xi)\|^2\right)^{1/2},
\end{equation}
where $F(\xi)\sim\mu$ and $G(\xi)\sim\nu$ share the same random input $\xi$.

We use this shared-noise coupling in \Cref{app:wasserstein} to relate averaged-velocity approximation to one-step sample error.

\subsection{Notation}
\label{app:notation}

\Cref{tab:notation} collects the method and analysis notation. Vector and feature norms are Euclidean/Frobenius; linear-map norms are spectral. Training losses use coordinate means, while transport bounds use joint norms.

{\small
\begin{longtable}{p{0.35\textwidth}p{0.57\textwidth}}
\caption{Unified notation for CoFlow.}\label{tab:notation}\\
\toprule Symbol & Meaning \\
\midrule
\endfirsthead
\multicolumn{2}{l}{\tablename\ \thetable\ (continued)}\\
\toprule Symbol & Meaning \\
\midrule
\endhead
\midrule\multicolumn{2}{r}{Continued on next page}\\
\endfoot
\bottomrule
\endlastfoot
\multicolumn{2}{l}{\textbf{Background and multi-agent state}} \\
$N$; $i,j$; $d$ & Number of agents; agent indices; flattened trajectory dimension per agent. \\

$q$; $r,t,s$; $h=t-r$ & Environment decision index; flow times with $r\le t$; interval length. \\
$\mathcal D$; $p_\tau$; $p_X$ & Offline dataset; clean observation-trajectory law; law of random variable $X$. \\
$\boldsymbol\tau_0=(\tau_0^i)_{i=1}^N$ & Clean joint observation trajectory, with $\tau_0^i\in\mathbb R^d$. \\
$\mathbf o_q$; $\mathbf a_q$ & Joint observations and actions; $o_q^i,a_q^i$ are agent blocks. \\
$\mathbf z_1$; $\mathbf z_t$ & Joint Gaussian noise in $\mathbb R^{Nd}$; sampled linear interpolant. \\
$\mathbf x_t$; $\Phi_{s\leftarrow t}$ & Marginal ODE trajectory; transport from flow time $t$ to $s$. \\
$\mathbf v_c$; $\mathbf v$ & Sample-conditional velocity $\mathbf z_1-\boldsymbol\tau_0$; theoretical marginal instantaneous field. \\
$\mathbf u$; $\mathbf u_\theta$; $\mathbf v_\theta$ & Theoretical averaged field; learned averaged field; background/base instantaneous predictor. \\
$u_\theta^i=[\mathbf u_\theta]^i$ & Agent-$i$ output block of the joint field, depending on the full team input; the same block convention applies to other joint vectors. \\
\addlinespace
\multicolumn{2}{l}{\textbf{Architecture and conditioning}} \\
$\theta$; $\phi$ & All velocity-network parameters, including gates; inverse-dynamics parameters. \\
$L,l$; $T_l,F_l$ & Number and index of CVA layers; temporal length and feature width at layer $l$. \\
$\boldsymbol h_l$; $c_l^i$; $\tilde h_l^i$ & Joint skip features; CVA message; enhanced skip. Subscript $l$ identifies the layer, and the tilde marks gated enhancement. \\
$\alpha\in[0,1]$ & CVA intervention multiplier in the experiments; distinct from attention weights $\alpha_{lm}^{ij}$. \\
$\gamma_l$; $\boldsymbol\gamma$ & Learnable scalar gate; collection of gates included in $\theta$. \\
$\mathsf c,\mathsf c_q$; $\mathsf c_q^{(i)}$; $m_i$ & Observation context; acting-agent-$i$ local context; teammate-observation masking operator. \\
$\mathcal H_q$; $C$ & Observation window $[\mathbf o_{q-C},\ldots,\mathbf o_q]$; number of past decisions, giving $C+1$ observations. \\
$R_{\mathrm{random}},R_{\mathrm{expert}}$ & Random-policy and expert-return anchors used for the MPE score normalization. \\
$R$; $\widetilde R$; $\varnothing$ & Target return; sampled training return condition; dropped return only (observation context retained). \\
$p_{\mathrm{drop}}$; $\omega$; $\hat{\mathbf u}_\theta$ & Return-drop probability; guidance weight; guided inference field. \\
\addlinespace
\multicolumn{2}{l}{\textbf{Training and sampling}} \\
$\rho$; $\mu_{\mathrm{time}},\sigma_{\mathrm{time}}$ & Fraction of time pairs with $r=t$; Gaussian parameters before logistic time transformation. \\
$\varepsilon_0$; $\varepsilon$ & Positive nominal finite-difference magnitude; valid nonzero signed shift when $r<t$. \\
$\mathrm{stopgrad}$ & Detaches the finite-difference or JVP correction; the data-derived tangent $\mathbf v_c$ is fixed. \\
$\mathcal D_{\mathbf v_c}$; $D_{\varepsilon}$ & Directional derivative along $(\mathbf v_c,1)$ at fixed $r$; finite-difference quotient. \\
$\boldsymbol\delta_{\mathrm{FD}}$; $\mathbf V_\theta$ & Detached joint derivative correction; corrected training prediction $\mathbf u_\theta+(t-r)\boldsymbol\delta_{\mathrm{FD}}$. \\
$\mathbf V_\theta^{\mathrm{JVP}}$ & Training prediction corrected by the stopped exact directional derivative. \\
$\mathcal L_{\mathrm{vel}},\mathcal L_{\mathrm{flow}},\mathcal L_{\mathrm{JVP}}$ & Velocity losses normalized by $Nd$; ordinary flow loss is the background/base objective. \\

$I_\phi^i$; $\ell_{\mathrm{act}}$; $\mathcal L_{\mathrm{act}}$ & Agent-$i$ inverse-dynamics predictor; per-example action loss; agent-averaged action objective. \\
$\mathcal U_q^i$; $d_a^i$; $y$ & Legal discrete actions; continuous-action dimension; decoder output vector/logits. \\
$\operatorname{Dec}_i$; $\lambda$; $\mathcal L$ & Action decoding operator; action-loss weight; combined training objective. \\
$\mathbf b,\mathbf b_+$; $\mathbf w_k$ & Base and shifted training evaluations; guided field evaluated before a simultaneous sampling update. \\
$\eta_{\mathrm{lr}}$; $K,k,t_k$ & Learning rate; sampling-step count, index, and time schedule. \\
$\hat{\boldsymbol\tau}_0$; $\hat o_{q+1}^i$ & Generated trajectory and predicted next observation (hat); $q$ indexes environment decisions and $0$ the data endpoint. \\
$\hat{\boldsymbol\tau}_0^{(i)}$; $\mathbf z_1^{(i)}$ & Joint prediction and noise in the local planning run of agent $i$; parentheses identify the run, not an output block. \\
\addlinespace
\multicolumn{2}{l}{\textbf{Bounds and mechanism analysis}} \\
$u_{\theta,0}^i$; $u_{\mathrm{ind}}^i,u_{\mathrm{coord}}^i$ & Gate-off intervention with other weights fixed; independent term and full gate-off difference. \\
$\bar{\mathbf u}_\theta$; $\bar u_\theta^i$; $\boldsymbol\Delta_{\mathrm{attn}}$ & Stacked self-only attention reference; its agent block; cross-agent residual. \\
$M,m$; $\alpha_{lm}^{ij}$; $B_{lm}$ & Attention-head count/index; row-normalized attention weights; effective value/output map in the idealized bound. \\
$g_l$; $f_l^i,\bar f_l^i$; $b_l^i$ & Shared non-expansive map in the bound; full/self-only features; injected cross-agent message. \\

$D_l$; $\delta_l^i$; $E_l$ & Feature diversity; full-minus-self-only feature perturbation; joint perturbation norm. \\
$\bar\sigma_l$; $\bar\sigma$ & Gate times summed head-map operator norms; maximum over layers. \\
$\mathcal I_{\varepsilon}$; $M_2^i,M_2$ & Signed query segment; per-agent and joint second-directional-derivative bounds. \\
$\mathbf R_\varepsilon$; $\mathbf g(a)$ & Joint Taylor remainder (distinct from return $R$); shifted joint query path. \\

$W_2$; $\mu,\nu$; $\Pi(\mu,\nu)$ & Wasserstein-2 distance; two laws; admissible couplings $\pi$ of those laws. \\

$\mathcal T=\Phi_{0\leftarrow1}$; $\boldsymbol\tau_{\mathrm{ref}}$ & Endpoint transport map; ideal sample coupled through the same noise. \\
$\bar{\mathbf u}=(\bar u^i)_{i=1}^N$; $\boldsymbol\Delta$ & Chosen oracle self-only terms and remaining joint residual. \\
$\epsilon_{\mathrm{vel}},\epsilon_i,\epsilon_{\mathrm{coord}}$ & Unnormalized endpoint root mean square (RMS) errors for the joint field, self-only blocks, and coordination residual. \\
$\bar\epsilon$; $D_N$; $\epsilon_{\mathrm{attn}}$ & Uniform self-only error bound; initial diversity for an $N$-agent team; assumed residual-error factor. \\
$\eta$ & Dataset-mean percentage threshold in \Cref{fig:steps_to_dataset_mean}; distinct from learning rate. \\
\addlinespace
\end{longtable}
}

\subsubsection*{Part II: Training and Inference Procedure}

\subsection{Implementation Notes}
\label{app:method_details_inline}

\smallskip\noindent\textbf{MeanFlow identity and joint dependence.} Writing out the arguments in \Cref{eq:meanflow_main} gives
\begin{equation}
\label{eq:meanflow_identity}
\mathbf u(\mathbf x_t,r,t)+(t-r)
\left[\nabla_{\mathbf x}\mathbf u(\mathbf x_t,r,t)\mathbf v(\mathbf x_t,t)
+\partial_t\mathbf u(\mathbf x_t,r,t)\right]
=\mathbf v(\mathbf x_t,t).
\end{equation}

For the fixed training tangent $\mathbf v_c$, the joint derivative and its agent-$i$ block are
\begin{equation}
\label{eq:jvp_consistency}
\begin{aligned}
\mathcal D_{\mathbf v_c}\mathbf u_\theta
&=(\nabla_{\mathbf z}\mathbf u_\theta)\mathbf v_c+\partial_t\mathbf u_\theta,\\
[\mathcal D_{\mathbf v_c}\mathbf u_\theta]^i
&=\sum_{j=1}^N\frac{\partial u_\theta^i}{\partial z^j}\,v_c^j+\partial_t u_\theta^i,
\end{aligned}
\end{equation}
where $\mathcal D_{\mathbf v_c}$ differentiates along $(\mathbf v_c,1)$ at fixed $r$ and context. All network terms use $(\mathbf z_t,r,t)$; the sum exposes dependence on teammate blocks.

The derivative-based reference predictor and its loss are
\begin{equation}
\begin{aligned}
\mathbf V_\theta^{\mathrm{JVP}}
&=\mathbf u_\theta+(t-r)\mathrm{stopgrad}[\mathcal D_{\mathbf v_c}\mathbf u_\theta],\\
\mathcal L_{\mathrm{JVP}}
&=\frac{1}{Nd}\mathbb E\|\mathbf V_\theta^{\mathrm{JVP}}-\mathbf v_c\|_2^2.
\end{aligned}
\end{equation}

This fits $\mathbf u_\theta$ to the stopped target $\mathbf v_c-(t-r)\mathcal D_{\mathbf v_c}\mathbf u_\theta$. Using a sampled tangent defines a stochastic target, not a pointwise identity for the theoretical field $\mathbf u$. For a fixed network, the conditional expectation of the un-stopped JVP residual recovers the marginal consistency residual because the derivative is linear in $\mathbf v_c$. Gradient stopping avoids differentiating the JVP again; its construction still has implementation-dependent compute and memory cost. \Cref{eq:fd_correction} approximates this joint derivative, with error bounded in \Cref{prop:fd_jvp}.

\smallskip\noindent\textbf{Finite-difference boundary and gradient conventions.} Use the same observation context, return condition, and stochastic masks for both queries. For a nonzero time interval, one valid signed-step rule is
\begin{equation}
\label{eq:fd_signed_step}
\varepsilon=\begin{cases}
\varepsilon_0,&t+\varepsilon_0\le1,\\
-\min(\varepsilon_0,h/2),&t+\varepsilon_0>1,
\end{cases}
\end{equation}
where $h=t-r>0$ and $\varepsilon_0>0$ is the nominal step magnitude.

This keeps the full segment inside $r\le t+\varepsilon\le1$ and avoids a zero denominator. At $h=0$, set the correction to zero without evaluating a quotient. The base query retains parameter gradients; its detached value is reused in the quotient, while the shifted query is evaluated without a gradient graph. The Taylor bound applies to any nonzero signed step satisfying this domain constraint.

\smallskip\noindent\textbf{CVA and parameter convention.} The $N$ feature streams share the temporal U-Net. Every CVA skip and the bottleneck use the gated operation in \Cref{eq:cva}. The decoder concatenates the enhanced skip with its current state before its residual blocks. All gates belong to $\theta$ and start at zero. The learned combined scale of the gates and projection matrices is defined in \Cref{thm:decomposition}.

\smallskip\noindent\textbf{Action loss and decoding.} The per-example loss and decoder are
\begin{equation}
\label{eq:action_decode_details}
\begin{aligned}
\ell_{\mathrm{act}}(a,y;\mathcal U)
&=\begin{cases}
\|a-y\|_2^2/d_a^i,&\text{continuous},\\
-y_a+\log\sum_{b\in\mathcal U}e^{y_b},&\text{discrete},
\end{cases}\\
\operatorname{Dec}_i(y;\mathcal U)
&=\begin{cases}y,&\text{continuous},\\
\arg\max_{b\in\mathcal U}y_b,&\text{discrete},
\end{cases}
\end{aligned}
\end{equation}
where $y=I_\phi^i(o_q^i,o_{q+1}^i)$, $d_a^i$ is the continuous-action dimension, and $\mathcal U=\mathcal U_q^i$.

The legal-action set is ignored in the continuous case. Discrete training actions belong to the nonempty legal set, and ties at decoding use a fixed rule. The expectation in \Cref{eq:action_loss} averages valid offline transitions, excluding padding or invalid transitions.

\smallskip\noindent\textbf{Observation context and execution modes.} We use the joint history window
\begin{equation}
\mathcal H_q=[\mathbf o_{q-C},\ldots,\mathbf o_q],
\end{equation}
where $C$ counts past decisions, giving $C+1$ observations; the joint context is $\mathsf c_q=\mathcal H_q$.

The window reduces to the current observation when $C=0$. For CoFlow-C, one joint generation uses $\mathsf c_q$. CoFlow-D uses a separate locally conditioned prediction for each acting agent. In the one-step case,
\begin{equation}
\begin{aligned}
\hat{\boldsymbol\tau}_0
&=\mathbf z_1-\hat{\mathbf u}_\theta(\mathbf z_1,0,1\mid\mathsf c_q,R),&&\text{CoFlow-C},\\
\hat{\boldsymbol\tau}_0^{(i)}
&=\mathbf z_1^{(i)}-\hat{\mathbf u}_\theta(\mathbf z_1^{(i)},0,1\mid\mathsf c_q^{(i)},R),&&\text{CoFlow-D},
\end{aligned}
\end{equation}
where $R$ is the target return and $\mathsf c_q^{(i)}=m_i(\mathcal H_q)$ masks unavailable teammate observations. Parenthesized $(i)$ identifies a local joint prediction, rather than an output block.

Agent $i$ decodes only its own output block. CoFlow-D shares CoFlow-C's trained weights and schedule. Within each locally conditioned joint prediction, cross-agent attention models teammate behavior by mixing generated trajectory features, without accessing unavailable teammate observations or hidden states.

\smallskip\noindent\textbf{Return guidance.} For every interval $[r,t]$, define
\begin{equation}
\label{eq:return_guidance}
\begin{aligned}
\hat{\mathbf u}_\theta(\mathbf z,r,t\mid\mathsf c,R)
&=\mathbf u_\theta(\mathbf z,r,t\mid\mathsf c,\varnothing)\\
&\quad+\omega\left[\mathbf u_\theta(\mathbf z,r,t\mid\mathsf c,R)
-\mathbf u_\theta(\mathbf z,r,t\mid\mathsf c,\varnothing)\right],
\end{aligned}
\end{equation}
where $\omega$ weights guidance and $\varnothing$ drops only the return condition, retaining observation context $\mathsf c$.

At $\omega=1$ this reduces to the conditional field. Training uses $\widetilde R=\varnothing$ with probability $p_{\mathrm{drop}}=0.25$, otherwise $\widetilde R=R$; both queries share that choice. Guidance may require two network evaluations per sampling step, so $K$ is not necessarily the evaluation count.

\subsection{Training and Inference Algorithms}
\label{app:algos}

The algorithms use the joint notation of the main text, with objective
\begin{equation}
\label{eq:joint_objective}
\min_{\theta,\phi}\quad\mathcal L(\theta,\phi)
=\frac{1}{Nd}\mathbb E\|\mathbf V_\theta-\mathbf v_c\|_2^2
+\frac{\lambda}{N}\sum_i\mathbb E\ell_{\mathrm{act}}\!\left(a_q^i,I_\phi^i(o_q^i,o_{q+1}^i);\mathcal U_q^i\right).
\end{equation}
Sample two independent logistic transforms of $\mathcal N(\mu_{\mathrm{time}},\sigma_{\mathrm{time}}^2)$ with $\mu_{\mathrm{time}}=-0.4$ and $\sigma_{\mathrm{time}}=1$, then sort them into $r\le t$. A fraction $\rho$ is set to $r=t$. In the algorithms, $\mathcal D$ denotes offline data, $\eta_{\mathrm{lr}}$ the learning rate, and $\boldsymbol\gamma$ the gate collection. The temporary predictions $\mathbf b,\mathbf b_+$ and $\mathbf w_k$ denote the base, shifted, and guided queries.

\begin{algorithm}[H]
\caption{CoFlow finite-difference training}
\label{alg:training}
\footnotesize
\begin{algorithmic}[1]
\REQUIRE Offline data $\mathcal D$, velocity model $\mathbf u_\theta$, inverse dynamics $I_\phi$, $\rho$, $\varepsilon_0$, $p_{\mathrm{drop}}$, $\lambda$, learning rate $\eta_{\mathrm{lr}}$
\STATE Initialize CVA gates $\boldsymbol\gamma=0$ as part of $\theta$
\FOR{each minibatch}
\STATE Sample joint trajectories, valid action transitions, legal-action sets, context $\mathsf c$, and return $R$ from $\mathcal D$
\STATE Draw joint noise $\mathbf z_1\sim\mathcal N(0,\mathbf I_{Nd})$ independently of the clean trajectory
\STATE Sample ordered time pairs $r\le t$ as above; set $r=t$ for a fraction $\rho$
\STATE Set $\widetilde R=\varnothing$ with probability $p_{\mathrm{drop}}$, and $\widetilde R=R$ otherwise
\STATE Build $\mathbf z_t=(1-t)\boldsymbol\tau_0+t\mathbf z_1$ and $\mathbf v_c=\mathbf z_1-\boldsymbol\tau_0$
\STATE Evaluate the joint base prediction $\mathbf b=\mathbf u_\theta(\mathbf z_t,r,t\mid\mathsf c,\widetilde R)$ with parameter gradients
\IF{$r<t$ (per example)}
\STATE Choose $\varepsilon$ by \Cref{eq:fd_signed_step}
\STATE Evaluate $\mathbf b_+=\mathbf u_\theta(\mathbf z_t+\varepsilon\mathbf v_c,r,t+\varepsilon\mid\mathsf c,\widetilde R)$ without gradients, keeping the same stochastic masks
\STATE Set $\boldsymbol\delta_{\mathrm{FD}}=\mathrm{stopgrad}[(\mathbf b_+-\mathrm{stopgrad}(\mathbf b))/\varepsilon]$
\ELSE
\STATE Set $\boldsymbol\delta_{\mathrm{FD}}=\mathbf0$ and skip the shifted query
\ENDIF
\STATE Set $\mathbf V_\theta=\mathbf b+(t-r)\boldsymbol\delta_{\mathrm{FD}}$
\STATE Compute minibatch estimates of \Cref{eq:velocity_loss,eq:action_loss} using the same normalization
\STATE Update $(\theta,\phi)\leftarrow(\theta,\phi)-\eta_{\mathrm{lr}}\nabla_{\theta,\phi}(\mathcal L_{\mathrm{vel}}+\lambda\mathcal L_{\mathrm{act}})$
\ENDFOR
\RETURN Trained $\mathbf u_\theta$ and $I_\phi$
\end{algorithmic}
\end{algorithm}

\begin{algorithm}[H]
\caption{CoFlow joint generation and action decoding}
\label{alg:inference}
\footnotesize
\begin{algorithmic}[1]
\REQUIRE Trained $\mathbf u_\theta,I_\phi$, available context, return $R$, legal actions, schedule $1=t_K>\cdots>t_0=0$, guidance $\omega$, execution mode
\STATE CoFlow-C performs one run with context $\mathsf c_q$ and decodes all agents
\STATE CoFlow-D performs the following run locally for each acting agent $i$, with context $\mathsf c_q^{(i)}$, and decodes only block $i$
\FOR{each required planning run, with its context denoted by $\mathsf c$}
\STATE Sample a joint noise tensor $\mathbf z_{t_K}\sim\mathcal N(0,\mathbf I_{Nd})$
\FOR{$k=K,K-1,\ldots,1$}
\STATE Evaluate the joint guided field $\mathbf w_k=\hat{\mathbf u}_\theta(\mathbf z_{t_k},t_{k-1},t_k\mid\mathsf c,R)$ using \Cref{eq:return_guidance}
\STATE Update the joint trajectory: $\mathbf z_{t_{k-1}}=\mathbf z_{t_k}-(t_k-t_{k-1})\mathbf w_k$
\ENDFOR
\STATE Set $\hat{\boldsymbol\tau}_0=\mathbf z_{t_0}$; for each agent $i$ executed by this run, read $\hat o_{q+1}^i$ from block $i$
\STATE Decode $a_q^i=\operatorname{Dec}_i(I_\phi^i(o_q^i,\hat o_{q+1}^i);\mathcal U_q^i)$ for the executed agents
\ENDFOR
\RETURN Joint action for CoFlow-C, or each acting agent's own action for CoFlow-D
\end{algorithmic}
\end{algorithm}

For $K=1$, the loop gives $\hat{\boldsymbol\tau}_0=\mathbf z_1-\hat{\mathbf u}_\theta(\mathbf z_1,0,1\mid\mathsf c,R)$, as in the main text. Sampling uses the averaged field alone, without the training correction in $\mathbf V_\theta$.

\noindent\textbf{CoFlow-base variant.} Removing the finite-difference target yields ordinary conditional-velocity regression, with predictor $\mathbf v_\theta(\mathbf z_t,t)\equiv\mathbf u_\theta(\mathbf z_t,0,t)$ and loss $\mathcal L_{\mathrm{flow}}$ in \Cref{eq:flow_loss}. The shared CVA backbone and inverse-dynamics head are unchanged. This ablation learns an instantaneous field and uses Euler integration during sampling.

\noindent\textbf{Remark on the surrogate.} \Cref{prop:fd_jvp} bounds the finite-difference target perturbation by $O((t-r)|\varepsilon|)$ under local smoothness. This is a controlled approximation to the JVP target, not a claim that the two optimization trajectories or final policies are identical.

\subsubsection*{Part III: Theoretical Justification}

\subsection{Taylor control of the finite-difference surrogate}
\label{app:fd_jvp_equiv}

The finite difference shifts the full joint state along $(\mathbf v_c,1)$ while holding $r$ and the conditioning context fixed. This includes every cross-agent Jacobian block in \Cref{eq:jvp_consistency}. The result below covers both signs of the step used in \Cref{eq:fd_correction}.

\begin{proposition}[Taylor control for the FD surrogate]
\label{prop:fd_jvp}
Fix $\mathbf z,\mathbf v_c\in\mathbb R^{Nd}$, $0\le r<t\le1$, and $h=t-r$. Let $\varepsilon\ne0$ satisfy $r\le t+\varepsilon\le1$, with $\mathcal I_\varepsilon=[\min(0,\varepsilon),\max(0,\varepsilon)]$. Assume $\mathbf u_\theta(\cdot,r,\cdot)$ is twice continuously differentiable along the query segment. Define
\begin{equation}
\label{eq:fd_directional}
D_\varepsilon\mathbf u_\theta(\mathbf z,r,t;\mathbf v_c)
=\frac{\mathbf u_\theta(\mathbf z+\varepsilon\mathbf v_c,r,t+\varepsilon)-\mathbf u_\theta(\mathbf z,r,t)}{\varepsilon},
\end{equation}
where $D_{\varepsilon}$ approximates $\mathcal D_{\mathbf v_c}$ by a finite difference.

The local curvature is bounded by
\begin{equation}
M_2^i=\sup_{a\in\mathcal I_{\varepsilon}}
\left\|\frac{d^2}{da^2}u_\theta^i(\mathbf z+a\mathbf v_c,r,t+a)\right\|_2,
\qquad M_2=\left(\sum_i(M_2^i)^2\right)^{1/2},
\end{equation}
where $M_2^i$ bounds the local second directional derivative and $M_2$ is the joint bound.

The directional derivative in \Cref{eq:jvp_consistency} satisfies
\begin{equation}
\label{eq:fd_jvp_error}
D_\varepsilon\mathbf u_\theta=\mathcal D_{\mathbf v_c}\mathbf u_\theta+\mathbf R_\varepsilon,
\qquad \|\mathbf R_\varepsilon\|_2\le\tfrac12|\varepsilon|M_2,
\end{equation}
where $\mathbf R_\varepsilon$ is the joint Taylor remainder. Each block obeys $\|R_\varepsilon^i\|_2\le\tfrac12|\varepsilon|M_2^i$.

Consequently, for the FD and JVP corrected predictions at the same query,
\begin{equation}
\label{eq:fd_surrogate_error}
\|\mathbf V_\theta-\mathbf V_\theta^{\mathrm{JVP}}\|_2
\le\tfrac12 h|\varepsilon|M_2.
\end{equation}
At $r=t$, both corrected predictions reduce to $\mathbf u_\theta$ and their difference is zero without any finite-difference evaluation.
\end{proposition}

\begin{proof}
Taylor's theorem for the joint query path gives, for either sign of $\varepsilon$,
\begin{equation}
\mathbf g(\varepsilon)=\mathbf g(0)+\varepsilon\mathbf g'(0)
+\varepsilon^2\int_0^1(1-s)\mathbf g''(s\varepsilon)\,ds,
\end{equation}
where $\mathbf g(a)=\mathbf u_\theta(\mathbf z+a\mathbf v_c,r,t+a)$. By the curvature definition, $\|\mathbf g''(a)\|_2\le M_2$ along the segment.

The chain rule gives $\mathbf g'(0)=\mathcal D_{\mathbf v_c}\mathbf u_\theta$. Dividing by $\varepsilon$ identifies
\begin{equation}
\mathbf R_\varepsilon=\varepsilon\int_0^1(1-s)\mathbf g''(s\varepsilon)\,ds,
\qquad \|\mathbf R_\varepsilon\|_2\le|\varepsilon|M_2\int_0^1(1-s)\,ds.
\end{equation}
This proves \Cref{eq:fd_jvp_error}; applying the same argument to block $i$ gives its bound. Gradient stopping leaves numerical values unchanged, and multiplication by $h$ gives \Cref{eq:fd_surrogate_error}.
\end{proof}

This is a per-sample target bound. If $\mathbb E[h^2|\varepsilon|^2M_2^2]<\infty$, its squared joint expectation is at most $\tfrac14\mathbb E[h^2|\varepsilon|^2M_2^2]$. Dividing by $Nd$ expresses this target-approximation bound in the coordinate-mean convention used for training.

\subsection{Joint Velocity Decomposition}
\label{app:decomposition}

\smallskip\noindent\textbf{Gate-off decomposition used for interventions.}
Let $\boldsymbol\gamma$ collect all learned CVA gates. The full model has the functional decomposition
\begin{equation}
\label{eq:gate_off_decomposition}
u_\theta^i(\mathbf z_t,r,t)
=\underbrace{u_{\theta,0}^i(z_t^i,r,t)}_{u_{\mathrm{ind}}^i}
+\underbrace{\left[
u_\theta^i(\mathbf z_t,r,t)
-u_{\theta,0}^i(z_t^i,r,t)\right]}_{u_{\mathrm{coord}}^i},
\end{equation}
where $u_{\theta,0}^i$ disables all gates while holding other weights fixed.

 The theorem below instead uses a self-only reference that retains the gated transformation of the agent's own value features. Its cross-agent term $\boldsymbol\Delta_{\mathrm{attn}}$ is therefore distinct from the entire gate-off difference $u_{\mathrm{coord}}$. The bound applies to $\boldsymbol\Delta_{\mathrm{attn}}$.

The gate-off identity above is algebraic. For the bound below, we analyze an idealized sequence of shared residual attention layers with non-expansive agent-local maps. The result applies under these architectural assumptions, with observation conditions and flow times fixed throughout the analysis.

\begin{theorem}[Joint velocity decomposition under non-expansive blocks]
\label{thm:decomposition}
Consider $L$ layers and $M$ heads with the shared feature recursion
\begin{equation}
\label{eq:attention_bound_model}
f_l^i=g_l\!\left(f_{l-1}^i+\gamma_l\sum_{m=1}^M\sum_{j=1}^N\alpha_{lm}^{ij}B_{lm}f_{l-1}^j\right),
\quad \alpha_{lm}^{ij}\ge0,\quad\sum_j\alpha_{lm}^{ij}=1,
\end{equation}
where $f_l^i$ are layer features, $\alpha_{lm}^{ij}$ is head-$m$ attention from $j$ to $i$, and $B_{lm}$ combines its value/output maps. The shared maps $g_l$ and the final output map are assumed non-expansive.

Input features are produced independently per agent. Define
\begin{equation}
\bar\sigma_l=|\gamma_l|\sum_m\|B_{lm}\|_{\mathrm{op}},
\qquad\bar\sigma=\max_l\bar\sigma_l,
\qquad D_l=\max_{i,j}\|f_l^i-f_l^j\|_2,
\end{equation}
where $\bar\sigma_l$ is the combined scale of the gates and projection matrices and $D_l$ the feature diversity, with input diversity $D_0$.

The network admits a self-only-plus-residual decomposition
\begin{equation}
\mathbf u_\theta(\mathbf z,r,t)
=\bar{\mathbf u}_\theta(\mathbf z,r,t)
+\boldsymbol\Delta_{\mathrm{attn}}(\mathbf z,r,t),
\end{equation}
where $\bar{\mathbf u}_\theta=(\bar u_\theta^i(z^i,r,t))_{i=1}^N$ stacks the self-only predictions and $\boldsymbol\Delta_{\mathrm{attn}}$ is the cross-agent residual.

The self-only reference replaces each attention row by weight one on its own agent, retaining the gated self transformation. Then
\begin{equation}
\|\boldsymbol\Delta_{\mathrm{attn}}\|_2
\le\frac{\sqrt N}{2}\left[(1+3\bar\sigma)^L-(1+\bar\sigma)^L\right]D_0.
\end{equation}
When $L\bar\sigma\ll1$, its leading-order expansion is $\sqrt N L\bar\sigma D_0$. The factor $\sqrt N$ disappears after per-agent normalization by $N^{-1/2}$.
\end{theorem}

\begin{proof}
Split the attention input into its self transformation and cross-agent message:
\begin{equation}
\begin{aligned}
b_l^i&=\gamma_l\sum_m\sum_j\alpha_{lm}^{ij}B_{lm}(f_{l-1}^j-f_{l-1}^i),\\
f_l^i&=g_l\!\left[\left(I+\gamma_l\sum_m B_{lm}\right)f_{l-1}^i+b_l^i\right],
\qquad \|b_l^i\|_2\le\bar\sigma_lD_{l-1},
\end{aligned}
\end{equation}
where $b_l^i$ is the injected cross-agent message.

The bound on $b_l^i$ follows from nonnegative row-normalized attention and the triangle inequality. Since the same $g_l$ is non-expansive for all agents,
\begin{equation}
D_l\le(1+\bar\sigma_l)D_{l-1}+2\bar\sigma_lD_{l-1}
\le(1+3\bar\sigma)^lD_0.
\end{equation}
Starting from $\bar f_0^i=f_0^i$, define the self-only reference recursively by
\begin{equation}
\bar f_l^i=g_l\!\left[\left(I+\gamma_l\sum_m B_{lm}\right)\bar f_{l-1}^i\right],
\quad\delta_l^i=f_l^i-\bar f_l^i,
\quad E_l=\left(\sum_i\|\delta_l^i\|_2^2\right)^{1/2},
\end{equation}
where $\bar f_l^i$ is the self-only feature and $E_l$ the joint perturbation norm.

With $E_0=0$, non-expansiveness and the joint norm of the injected messages imply
\begin{equation}
E_l\le(1+\bar\sigma_l)E_{l-1}+\sqrt N\bar\sigma_lD_{l-1}.
\end{equation}
Iterating and applying the non-expansive output map gives
\begin{equation}
\begin{aligned}
\|\boldsymbol\Delta_{\mathrm{attn}}\|_2
&\le\sqrt N\bar\sigma\sum_{l=1}^L(1+\bar\sigma)^{L-l}(1+3\bar\sigma)^{l-1}D_0\\
&=\frac{\sqrt N}{2}\left[(1+3\bar\sigma)^L-(1+\bar\sigma)^L\right]D_0.
\end{aligned}
\end{equation}
The zero-gate case follows by continuity. Expanding the bracket for $L\bar\sigma\ll1$ yields $2L\bar\sigma+O(L^2\bar\sigma^2)$. If $D_0=0$, the cross-agent correction is zero under the stated shared-map assumptions. Extending this analysis to expansive or agent-dependent maps requires incorporating their Lipschitz factors and cross-agent differences into the recurrences.
\end{proof}

The small-gating regime is determined by the combined scale of the gates and projection matrices $L\bar\sigma$.

\subsection{Single-Pass Wasserstein Error Bound}
\label{app:wasserstein}

We relate approximation of the averaged velocity in \Cref{eq:average_velocity} to the distribution of one-step samples. All observation and return conditions are fixed and suppressed. This result concerns the unguided averaged-velocity predictor; it also applies to an effective guided field if that field satisfies the same approximation assumption for the chosen target transport.

\begin{assumption}[Marginal transport and endpoint approximation]
\label{asm:approx}
The marginal flow admits a measurable endpoint map $\mathcal T=\Phi_{0\leftarrow1}$ such that $\mathcal T(\mathbf z_1)$ has law $p_\tau$ for $\mathbf z_1\sim\mathcal N(0,\mathbf I)$. The relevant distributions have finite second moments, and the ideal endpoint averaged velocity is
\begin{equation}
\mathbf u(\mathbf z_1,0,1)=\mathbf z_1-\mathcal T(\mathbf z_1).
\end{equation}

Define the network's endpoint approximation error by
\begin{equation}
\epsilon_{\mathrm{vel}}^2
\coloneqq\mathbb E_{\mathbf z_1}
\|\mathbf u_\theta(\mathbf z_1,0,1)-\mathbf u(\mathbf z_1,0,1)\|^2<\infty,
\end{equation}
where $\epsilon_{\mathrm{vel}}$ is the unnormalized joint endpoint RMS error.

\end{assumption}

The oracle here is the displacement of the marginal transport, not the conditional mean of the independently paired training velocity.

\begin{theorem}[Single-pass Wasserstein error bound]
\label{thm:one_step_error}
Under \Cref{asm:approx}, the one-step sample
$\hat{\boldsymbol\tau}_0=\mathbf z_1-\mathbf u_\theta(\mathbf z_1,0,1)$ satisfies
\begin{equation}
W_2(p_{\hat{\boldsymbol\tau}_0},p_\tau)\le\epsilon_{\mathrm{vel}}.
\end{equation}

For any fixed per-agent-plus-residual decomposition of the oracle,
\begin{equation}
\epsilon_{\mathrm{vel}}
\le \left(\sum_{i=1}^N\epsilon_i^2\right)^{1/2}
+\epsilon_{\mathrm{coord}},
\end{equation}
where $\epsilon_i$ and $\epsilon_{\mathrm{coord}}$ are the RMS errors of the self-only terms and joint residual.

\end{theorem}

\begin{proof}
Set $\boldsymbol\tau_{\mathrm{ref}}=\mathcal T(\mathbf z_1)=\mathbf z_1-\mathbf u(\mathbf z_1,0,1)$, whose law is $p_\tau$ by assumption. Coupling $\hat{\boldsymbol\tau}_0$ and $\boldsymbol\tau_{\mathrm{ref}}$ through the same noise gives
\begin{equation}
\begin{aligned}
W_2(p_{\hat{\boldsymbol\tau}_0},p_\tau)
&\le\left(\mathbb E\|\hat{\boldsymbol\tau}_0-\boldsymbol\tau_{\mathrm{ref}}\|^2\right)^{1/2}\\
&=\left(\mathbb E\|\mathbf u_\theta-\mathbf u\|^2\right)^{1/2}
=\epsilon_{\mathrm{vel}},
\end{aligned}
\end{equation}
where omitted velocity arguments are $(\mathbf z_1,0,1)$.

For the second statement, write
\begin{equation}
\mathbf u_\theta-\mathbf u
=(\bar{\mathbf u}_\theta-\bar{\mathbf u})
+(\boldsymbol\Delta_{\mathrm{attn}}-\boldsymbol\Delta),
\end{equation}
where $\bar{\mathbf u}=(\bar u^i)_{i=1}^N$ and $\boldsymbol\Delta$ are the oracle's chosen self-only terms and joint residual, evaluated at the same endpoint.

Define
\begin{equation}
\epsilon_i^2=\mathbb E\|\bar u_\theta^i-\bar u^i\|^2,
\qquad
\epsilon_{\mathrm{coord}}^2
=\mathbb E\|\boldsymbol\Delta_{\mathrm{attn}}-\boldsymbol\Delta\|^2.
\end{equation}
The triangle inequality in $L^2$, followed by orthogonality of the agent output blocks, yields
\begin{equation}
\begin{aligned}
\epsilon_{\mathrm{vel}}
&\le\left(\mathbb E\|\bar{\mathbf u}_\theta-\bar{\mathbf u}\|^2\right)^{1/2}
+\epsilon_{\mathrm{coord}}\\
&=\left(\sum_i\epsilon_i^2\right)^{1/2}
+\epsilon_{\mathrm{coord}}.
\end{aligned}
\end{equation}
\end{proof}

\smallskip\noindent\textbf{Relation to finite-difference training.}
\Cref{prop:fd_jvp} bounds the local approximation error of the finite-difference training target, while \Cref{thm:one_step_error} bounds one-step sample error given endpoint averaged-velocity accuracy. Linking the training loss to endpoint accuracy requires control of the time distribution, approximation, and optimization. Together, the results characterize target approximation and conditional error propagation.

\subsection{Scaling with the number of agents}
\label{app:scaling}

Assume $\epsilon_i\le\bar\epsilon$ and that $\epsilon_{\mathrm{coord}}$ is bounded by the architectural envelope times $\epsilon_{\mathrm{attn}}$. Under the assumptions of \Cref{thm:decomposition} and its small-gating expansion,
\begin{equation}
\epsilon_{\mathrm{vel}} \lesssim \sqrt{N}\,\bar\epsilon + \sqrt{N}\,L\,\bar\sigma\,D_N\,\epsilon_{\mathrm{attn}}
\qquad (L\bar\sigma \ll 1),
\end{equation}
where $D_N=\max_{i\ne j}\|f_0^i-f_0^j\|$ is the initial diversity for $N$ agents.

The leading $\sqrt N$ factor is a consequence of measuring a joint vector in Euclidean norm; it disappears after per-agent normalization. Thus coordination scalability depends on both team size and the controlled growth of feature diversity $D_N$ and learned gate scale $\bar\sigma$.


\end{document}